\documentclass{article}

\usepackage{microtype}
\usepackage{graphicx}
\usepackage{booktabs} 

\usepackage{hyperref}

\usepackage[accepted]{icml2024}

\usepackage{amsmath}
\usepackage{amssymb}
\usepackage{mathtools}
\usepackage{amsthm}

\usepackage[capitalize,noabbrev]{cleveref}

\usepackage{dsfont}
\usepackage[caption=false,font=normalsize,labelfont=sf,textfont=sf]{subfig}
\usepackage{multirow}

\theoremstyle{plain}
\newtheorem{theorem}{Theorem}[section]

\newtheorem{lemma}[theorem]{Lemma}
\newtheorem{corollary}[theorem]{Corollary}
\newcounter{customcounter}
\newtheorem{theoremappendix}{Theorem}[customcounter]
\newtheorem{lemmaappendix}[theoremappendix]{Lemma}
\newtheorem{corollaryappendix}[theoremappendix]{Corollary}
\theoremstyle{definition}

\theoremstyle{remark}

\usepackage[textsize=tiny]{todonotes}

\DeclareMathOperator{\E}{\mathbb{E}}
\DeclareMathOperator*{\argmax}{arg\,max}

\icmltitlerunning{$f$-Divergence Based Classification: Beyond the Use of Cross-Entropy}

\begin{document}

\twocolumn[
\icmltitle{$f$-Divergence Based Classification: Beyond the Use of Cross-Entropy}



\icmlsetsymbol{equal}{*}

\begin{icmlauthorlist}
\icmlauthor{Nicola Novello}{yyy}
\icmlauthor{Andrea M. Tonello}{yyy}
\end{icmlauthorlist}

\icmlaffiliation{yyy}{Department of Networked and Embedded Systems, University of Klagenfurt, Klagenfurt, Austria}

\icmlcorrespondingauthor{Nicola Novello}{nicola.novello@aau.at}

\icmlkeywords{Machine Learning, ICML}

\vskip 0.3in
]



\printAffiliationsAndNotice{}  

\begin{abstract}
In deep learning, classification tasks are formalized as optimization problems often solved via the minimization of the cross-entropy.  
However, recent advancements in the design of objective functions allow the usage of the $f$-divergence to generalize the formulation of the optimization problem for classification. 
We adopt a Bayesian perspective and formulate the classification task as a maximum a posteriori probability problem. 
We propose a class of objective functions based on the variational representation of the $f$-divergence. Furthermore, driven by the challenge of improving the state-of-the-art approach, we propose a bottom-up method that leads us to the formulation of an objective function corresponding to a novel $f$-divergence referred to as shifted log (SL). 
We theoretically analyze the objective functions proposed and numerically test them in three application scenarios: toy examples, image datasets, and signal detection/decoding problems. The analyzed scenarios demonstrate the effectiveness of the proposed approach and that the SL divergence achieves the highest classification accuracy in almost all the considered cases.
\end{abstract}

\section{Introduction}
\label{sec:introduction}


Classification problems are relevant in a multitude of domains, such as computer vision, biomedical, and telecommunications engineering \cite{peng2010novel, DeepLearningMethodsforImprovedDecodingofLinearCodes, uy2019revisiting}. In general, classification refers to the estimation of a discrete vector $\mathbf{x}$ (i.e., the class) given an observation vector $\mathbf{y}$.
In the Bayesian framework, the optimal method to solve classification problems is derived from the maximum a posteriori probability (MAP) principle \cite{boyd2004convex, jeong2024demystifying}. 
Classical estimation theory uses a model-based approach to derive the estimation algorithm. Then, the MAP algorithm is well-defined and applicable when the posterior density is known. 
If the posterior probability is unknown, the first fundamental step towards solving the MAP problem consists of learning the posterior density from the data. In this direction, deep learning (DL) approaches learn the model and solve the estimation task directly via data observation. This is achieved by leveraging artificial neural networks, whose ability to model probability density functions (pdfs) makes them particularly suited for this task, as shown in \cite{hornik1989multilayer, lecun2015deep, mohamed2016learning, papamakarios2017masked}. DL models require the design of two main elements: the network architecture, which defines the class of functions the network can estimate, and the objective function that is exploited during the training phase to learn the optimal parameters of the network.  
In reference to neural network-based classification techniques, most of the previous work focused on the conceievement of the network architecture \cite{DBLP:journals/corr/SimonyanZ14a, szegedy2015going, rezende2017malicious, tong2020channel, bhojanapalli2021understanding}. 
Contrarily, a smaller part of classification literature is dedicated to the objective function design. In most cases, classification is achieved through the minimization of the cross-entropy loss function between the data empirical probability density function $p_{data}$ and the probability density function output of the neural network $p_{model}$ \cite{kussul2017deep, HyperspectralImageClassificationWithDeepLearningModels}. 
The minimization of the cross-entropy corresponds to the minimization of the Kullback-Leibler (KL) divergence between the same two probability distributions. The KL divergence is defined as 
\begin{equation}
\label{eq:DKL definition}
    D_{KL}(p_{data} || p_{model}) = \E_{p_{data}} \left[ \log \left( \frac{p_{data}}{p_{model}} \right) \right] ,
\end{equation}
while the cross-entropy (CE) is $\E_{p_{data}} \left[- \log \left( p_{model} \right) \right]$. \\
Alternatively, some papers propose the usage of proper losses, which is a family of losses characterized by Bregman divergences \cite{gneiting2007strictly}. For instance, the authors in \cite{hui2020evaluation} compare the square loss and the cross-entropy for supervised classification tasks, while in \cite{dong2019single}, the authors propose two novel objective functions based on logistic regression.\\
Another popular class of divergence functions is the $f$-divergence, which has been used for various classification algorithms. The authors in \cite{yu2020training} propose a min-max game for the objective function design of deep energy-based models, where they substitute the minimization of the KL with any $f$-divergence. In \cite{wei2020optimizing}, the authors propose a max-max optimization problem to tackle classification with noisy labels. They maximize the $f$-mutual information (a generalization of the mutual information) between the classifier's output and the true label distribution. 
In \cite{zhong2023learning}, the $f$-divergence is used as a regularization term in a min-max optimization problem, for the design of fair classifiers (i.e., minimize the classifier discrepancy over sub-groups of the population). 

In this paper, we propose to estimate the conditional posterior probability (needed in the MAP classifier) by expressing it as a density ratio (see \eqref{eq:top_down_posterior_estimator}) and by using a discriminative learning approach \cite{Song2020}. 
Density ratio estimation approaches have been used in a wide variety of applications \cite{nowozin2016f, wei2020optimizing, letizia2023variational}. 
However, unlike known classification objective functions, the main idea behind the proposed estimators is not to use a divergence minimization-based technique. Instead, to solve the classification task, we estimate the ratio between joint and marginal probability densities with a discriminator network and maximize such a ratio (corresponding to the posterior probability) with respect to the class elements. Therefore, contrarily to other $f$-divergence-based approaches that need a double training optimization procedure (e.g., max-max or min-max), our approach relies on a single training maximization formulation needed to learn the posterior density. 
Additionally, it enables the use of the $f$-divergence to obtain a broader set of classifiers beyond the conventional approach based on the exploitation of the cross-entropy. 

In more detail, the contributions of this paper are fourfold: a) we design a class of posterior probability estimators that exploits the variational representation of the $f$-divergence \cite{Nguyen2010}; b) we present a second class of posterior probability estimators formulated using a bottom-up approach; c) we propose a new objective function for classification tasks that corresponds to the variational representation of a novel $f$-divergence. The proposed divergence is analyzed theoretically and experimentally, and a comparison with the $f$-divergences known in the literature follows, showing that the new one achieves the best performance in almost all the considered scenarios; d) finally, we conceive two specific network architectures that are trained with the proposed objective functions.

\section{MAP-based Classification Through Posterior Probability Learning}
\label{sec:map}

\begin{figure}[htbp]
	\centerline{\includegraphics[width=0.9\columnwidth]{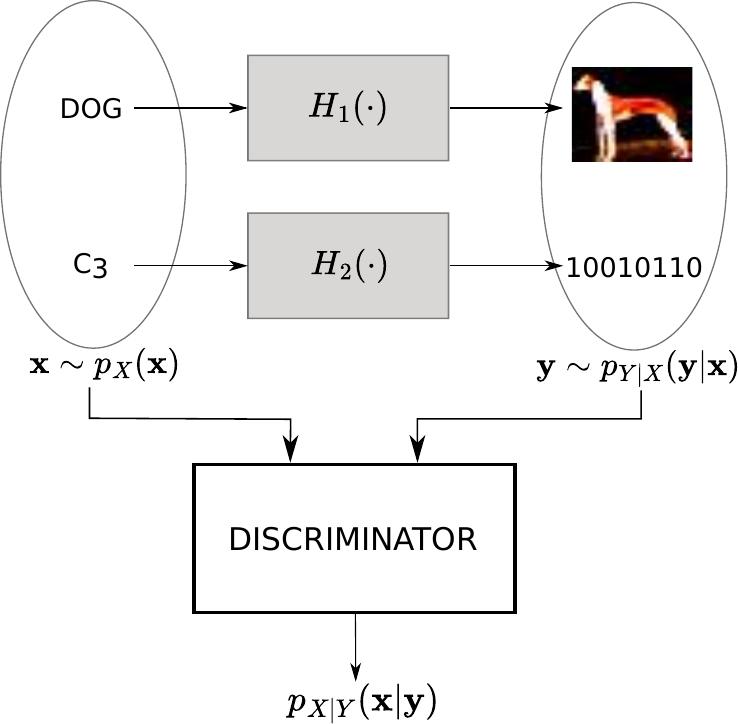}}
	\caption{System model representation. $X$ is the input of a stochastic model $H(\cdot)$, while the output is $Y$. In the example represented by $H_1(\cdot)$, the input is the class element "dog" and the output is an image of a dog. Differently, $H_2(\cdot)$ is a communication channel, therefore the input is a codeword, and the output is the binary representation of such a codeword after the noise addition.}
	\label{fig:system-model}
\end{figure}

In this section, we describe the approach that we propose to tackle classification tasks and the notation used.
Let $X$ and $Y$ be two random vectors described by the probability density functions $p_X(\mathbf{x})$ and $p_Y(\mathbf{y})$, respectively. Let $X$ and $Y$ be the input and output of a stochastic model (referred to as $H(\cdot)$), respectively, as shown in Fig. \ref{fig:system-model}. 
In a classification context, $X$ is discrete and represents the class type with alphabet $\mathcal{A}_x$, while $Y$ is the observation of the class elements. \\ 
The MAP estimator (classifier) is formulated as: 
\begin{equation}
\label{eq:MAP_problem}
    \hat{\mathbf{x}} = \argmax_{\mathbf{x} \in \mathcal{A}_x} p_{X|Y}(\mathbf{x}|\mathbf{y}) ,
\end{equation}
where $p_{X|Y}(\mathbf{x}|\mathbf{y})$ is the posterior probability density.
In this paper, we propose to adopt a discriminative formulation. Therefore, we express $p_{X|Y}(\mathbf{x}|\mathbf{y})$ as the ratio between $p_{XY}(\mathbf{x}, \mathbf{y})$ and $p_Y(\mathbf{y})$ by using the Bayes theorem.\\
Since the classification task is an estimation problem, we formulate it with the MAP approach because it is the optimal approach for estimation problems in the Bayesian framework \cite{proakis2007fundamentals, jeong2024demystifying}. A key advantage of the proposed MAP-based approach is that it does not need a double training optimization procedure because it directly learns the posterior pdf formulating the objective function either using the $f$-divergence (Section \ref{sec:topDown}) or with a bottom-up approach (Section \ref{sec:bottom-up}). When the posterior pdf is learned, the estimation (thus the classification) problem is solved, because the argmax operator in \eqref{eq:MAP_problem} can be easily computed. 
First, we study the more general problem of estimating the posterior pdf when $X$ is a continuous random vector. Then, we consider the specific case of $X$ being discrete. The classification problem is solved choosing the optimal class $\mathbf{x}$ that, during inference, maximizes the estimate of $p_{X|Y}(\mathbf{x}|\mathbf{y})$ (i.e., that solves \eqref{eq:MAP_problem}).

\section{Posterior Probability Learning Through the Exploitation of $f$-Divergence}
\label{sec:topDown}

The first approach we propose is to estimate $p_{X|Y}(\mathbf{x}|\mathbf{y})$ by exploiting the variational representation of the $f$-divergence. 

\subsection{$f$-Divergence}
Given two probability distributions $P$ and $Q$ admitting, respectively, the absolutely continuous density functions $p$ and $q$ with respect to $d\mathbf{x}$ defined on the domain $\mathcal{X}$, the $f$-divergence (also known as Ali-Silvey distance) is defined as \cite{ali1966general, csiszar1967information} 
\begin{equation}
    D_f(P||Q) = \int_{\mathcal{X}} q(\mathbf{x}) f\left( \frac{p(\mathbf{x})}{q(\mathbf{x})} \right) d\mathbf{x} ,
\end{equation}
where the \textit{generator function} $f: \mathbb{R}_+ \longrightarrow \mathbb{R}$ is a convex, lower-semicontinuous function such that $f(1)=0$. 
Every generator function has a \textit{Fenchel conjugate} function $f^*$, that is defined as 
\begin{equation}
\label{eq:fenchel_conjugate}
    f^*(t) = \sup_{u \in dom_f} \left\{ ut - f(u) \right\},
\end{equation}
where $dom_f$ is the domain of $f(\cdot)$, $f^{*}$ is convex and such that $f^{**}(u) = f(u)$. \\
Leveraging \eqref{eq:fenchel_conjugate}, the authors in \cite{Nguyen2010} expressed a lower bound on any $f$-divergence, that is referred to as variational representation of the $f$-divergence:
\begin{align}
\label{eq:variational_representation}
    D_f(P||Q) &\geq \sup_{T \in \mathcal{T}} \left\{ \E_{p(\mathbf{x})} \left[ T(\mathbf{x}) \right] - \E_{q(\mathbf{x})} \left[ f^*(T(\mathbf{x})) \right] \right\} ,
\end{align}
where $T(\mathbf{x})$ is parametrized by an artificial neural network, and the bound in \eqref{eq:variational_representation} is tight when $T(\mathbf{x})$ is
\begin{equation}
\label{eq:T_hat}
    T^{\diamond}(\mathbf{x}) = f^{\prime} \left( \frac{p(\mathbf{x})}{q(\mathbf{x})} \right) ,
\end{equation}
where $f^\prime$ is the first derivative of $f$.


\subsection{Posterior Estimation Through the Variational Representation of the $f$-Divergence}
\label{subsec:top-down}
Theorem \ref{theorem:top_down} provides a class of objective functions that, when maximized, leads to the estimation of the posterior probability density.

\begin{theorem}
\label{theorem:top_down}
Let $X$ and $Y$ be the random vectors with pdfs $p_X(\mathbf{x})$ and $p_Y(\mathbf{y})$, respectively. Assume $\mathbf{y} = H(\mathbf{x})$, where $H(\cdot)$ is a stochastic function, then $p_{XY}(\mathbf{x}, \mathbf{y})$ is the joint density. Define $\mathcal{T}_x$ to be the support of $X$ and $p_U(\mathbf{x})$ to be a uniform pdf with support $\mathcal{T}_x$. Let $f_{u}: \mathbb{R}_+ \longrightarrow \mathbb{R}$ be a convex function such that $f_{u}(1)=0$, and $f_{u}^*$ be the Fenchel conjugate of $f_{u}$.
Let $\mathcal{J}_f(T)$ be the objective function defined as 
\begin{align}
\label{eq:top_down_variational_representation_posterior}
    \mathcal{J}_f(T) &= \E_{(\mathbf{x},\mathbf{y}) \sim p_{XY}(\mathbf{x},\mathbf{y})} \left[ T(\mathbf{x},\mathbf{y}) \right] \notag \\
    & - \E_{(\mathbf{x},\mathbf{y}) \sim p_U(\mathbf{x})p_Y(\mathbf{y})} \left[ f_{u}^*(T(\mathbf{x},\mathbf{y})) \right] .
\end{align}
Then, 
\begin{equation}
    T^{\diamond}(\mathbf{x},\mathbf{y}) = \argmax_{T \in \mathcal{T}} \mathcal{J}_f(T) 
\end{equation}
leads to the estimation of the posterior density
\begin{equation}
\label{eq:top_down_posterior_estimator}
    \hat{p}_{X|Y}(\mathbf{x}|\mathbf{y}) = \frac{p_{XY}(\mathbf{x},\mathbf{y})}{p_Y(\mathbf{y})} = \frac{(f_{u}^{*})^{\prime}(T^{\diamond}(\mathbf{x},\mathbf{y}))}{|\mathcal{T}_x|},
\end{equation}
where $T^{\diamond}(\mathbf{x},\mathbf{y})$ is parametrized by an artificial neural network.
\end{theorem}
\setcounter{table}{0}
\begin{table*}[t]
  \centering
  \caption{Objective functions table. The corresponding $f$-divergences are: Kullback-Leibler, Reverse Kullback-Leibler, squared Hellinger distance, GAN, and Pearson $\chi^2$.}
  \label{tab:value functions}
  \vskip 0.15in
  \begin{center}
  \begin{small}
  \begin{sc}
  \begin{tabular}{ c c c c } 
     \hline
     Name & Objective function & $T^{\diamond}(\mathbf{x}, \mathbf{y})$ & $D^{\diamond}(\mathbf{x}, \mathbf{y})$\\
     \hline
      $\mathcal{J}_{KL}(D)$ & $\E_{p_{XY}(\mathbf{x}, \mathbf{y})} \Bigl[ \log(D(\mathbf{x}, \mathbf{y})) \Bigr] - \E_{p_U(\mathbf{x})p_Y(\mathbf{y})}\Bigl[ |\mathcal{T}_x| D(\mathbf{x}, \mathbf{y}) \Bigr] $ & $\log(D^{\diamond}(\mathbf{x}, \mathbf{y})) +1$ & $\frac{p_{XY}(\mathbf{x}, \mathbf{y})}{p_Y(\mathbf{y})}$\\
      $\mathcal{J}_{RKL}(D)$ & $\E_{p_{XY}(\mathbf{x}, \mathbf{y})} \Bigl[ - D(\mathbf{x}, \mathbf{y}) \Bigr] + \E_{p_U(\mathbf{x})p_Y(\mathbf{y})} \Bigl[ |\mathcal{T}_x| \log(D(\mathbf{x}, \mathbf{y})) \Bigr] $ & $-D^{\diamond}(\mathbf{x}, \mathbf{y})$ & $\frac{p_Y(\mathbf{y})}{p_{XY}(\mathbf{x}, \mathbf{y})}$ \\
      $\mathcal{J}_{HD}(D)$ & $\E_{p_{XY}(\mathbf{x}, \mathbf{y})} \Bigl[- \sqrt{D(\mathbf{x}, \mathbf{y})} \Bigr] - \E_{p_U(\mathbf{x})p_Y(\mathbf{y})} \Bigl[ |\mathcal{T}_x| \frac{1}{\sqrt{D(\mathbf{x}, \mathbf{y})}} \Bigr] $ & $1 - \sqrt{D^{\diamond}}$ & $\frac{p_Y(\mathbf{y})}{p_{XY}(\mathbf{x}, \mathbf{y})}$ \\
      $\mathcal{J}_{GAN}(D)$ & $\E_{p_{XY}(\mathbf{x}, \mathbf{y})} \Bigl[\log(1 - D(\mathbf{x}, \mathbf{y}) )\Bigr]+ \E_{p_U(\mathbf{x})p_Y(\mathbf{y})} \Bigl[ |\mathcal{T}_x| \log(D(\mathbf{x}, \mathbf{y})) \Bigr]$ & $\log(1 - D^{\diamond}(\mathbf{x}, \mathbf{y}))$ & $\frac{p_Y(\mathbf{y})}{p_{XY}(\mathbf{x}, \mathbf{y}) + p_Y(\mathbf{y})}$ \\
      $\mathcal{J}_{P}(D)$ & $\E_{p_{XY}(\mathbf{x}, \mathbf{y})} \Bigl[2(D(\mathbf{x}, \mathbf{y}) - 1)\Bigr] - \E_{p_U(\mathbf{x})p_Y(\mathbf{y})} \Bigl[ |\mathcal{T}_x| D^2(\mathbf{x}, \mathbf{y}) \Bigr]$ & $2(D^{\diamond}(\mathbf{x}, \mathbf{y}) - 1)$ & $\frac{p_{XY}(\mathbf{x}, \mathbf{y})}{p_Y(\mathbf{y})}$ \\
     \hline
    \end{tabular}
    \end{sc}
    \end{small}
    \end{center}
    \vskip -0.1in
\end{table*}
According to Theorem \ref{theorem:top_down}, there exists a class of objective functions to train a discriminator whose output is processed to obtain an estimate of the posterior probability.
The choice of the $f$-divergence offers a degree of freedom (DOF) in the objective function design.
To improve the training convergence of the objective functions formulated as variational representation of $f$-divergences, the literature exploits a change of variable $T(\mathbf{x}, \mathbf{y}) = r(D(\mathbf{x}, \mathbf{y}))$ in \eqref{eq:top_down_variational_representation_posterior} \cite{nowozin2016f}. With the same goal, we propose to introduce a second DOF in the development of the objective function \eqref{eq:top_down_variational_representation_posterior}. 
Accordingly, it is sufficient to substitute $D^{\diamond}(\mathbf{x}, \mathbf{y}) = r^{-1}(T^{\diamond}(\mathbf{x}, \mathbf{y}))$ in \eqref{eq:top_down_posterior_estimator} to attain the corresponding posterior probability estimator. The exploitation of the DOFs uniquely defines the objective function and numerically impacts the discriminator's parameters convergence during the training phase. 
We use Theorem \ref{theorem:top_down} to obtain five specific estimators trained with the objective functions reported in Tab. \ref{tab:value functions}. These objective functions are obtained by first selecting the generator function and then choosing $r(\cdot)$. The generator functions used to derive Tab. \ref{tab:value functions} are reported in Tab. \ref{tab:f_divergences_table_unsupervised} of Appendix \ref{subsec:Appendix_value_functions} and are retrieved from a modification of the well known $f$-divergences reported for completeness in Tab. \ref{tab:f_divergences_table} of Appendix \ref{subsec:Appendix_value_functions}. In particular, the modification applied is attained as $f_{u}^{*}(t) = |\mathcal{T}_x|f^{*}(t)$, where $f_u^{*}(t)$ and $f^{*}(t)$ are referred to as \textit{unsupervised} and \textit{supervised} generator functions, respectively. In addition, Tab. \ref{tab:value functions} comprises the change of variable $r(D^{\diamond}(\mathbf{x}, \mathbf{y}))$ and the expression of $D^{\diamond}(\mathbf{x}, \mathbf{y})$ for each objective function.

Lemma \ref{lemma:convergence} proves the convergence property of any posterior probability estimator that is formulated as in Theorem \ref{theorem:top_down}. 

\begin{lemma}
\label{lemma:convergence}
    Let the artificial neural network $D(\cdot) \in \mathcal{D}$ be with enough capacity and training time (i.e., in the nonparametric limit). Assume the gradient ascent update rule $D^{(i+1)} = D^{(i)} + \mu \nabla \mathcal{J}_f(D^{(i)})$ converges to
    \begin{equation}
        D^{\diamond} = \argmax_{D \in \mathcal{D}} \mathcal{J}_f(D) ,
    \end{equation}
    where $\mathcal{J}_f(D)$ is defined as in \eqref{eq:top_down_variational_representation_posterior}, with the change of variable $D = r^{-1}(T)$.
    Then, the difference between the optimal posterior probability and its estimate at iteration $i$ is
    \begin{align}
        p^{\diamond} - p^{(i)} \simeq \frac{1}{|\mathcal{T}_x|} \Bigl( \delta^{(i)} \Bigl[ (f_{u}^{*})^{\prime \prime}(r(D^{(i)})) \Bigr] \Bigr) , \label{eq:gap_optimum_posterior}
    \end{align}
    where $\delta^{(i)}=r(D^{\diamond}) - r(D^{(i)})$, $(f_{u}^{*})^{\prime \prime}$ is the second derivative of $f_u^*$, and $\mu >0$ the learning rate.
    If $D^{\diamond}$ corresponds to the global optimum achieved by using the gradient ascent method, the posterior probability estimator in \eqref{eq:top_down_posterior_estimator} converges to the real value of the posterior density.
\end{lemma}

In addition to proving the convergence property of the posterior estimators in Theorem \ref{theorem:top_down}, Lemma \ref{lemma:convergence} provides an intuitive explanation of the posterior probability estimator's bias's dependency on $f$ in \eqref{eq:gap_optimum_posterior}.

Theorem \ref{theorem:top_down} provides an effective method to solve classification problems by designing the objective function based on the choice of an $f$-divergence. This differs from other methods that leverage $f$-divergences and that need a dual optimization strategy \cite{wei2020optimizing, zhong2023learning}. The proposed method relies on a single optimization problem. In fact, it generalizes the cross-entropy minimization approach (see Section \ref{sec:Architecture}). 
In the next section, we propose a novel bottom-up approach that guides the design of new objective functions for the posterior probability estimation problem.

\section{Bottom-Up Posterior Probability Learning}
\label{sec:bottom-up}

In this section, we propose a bottom-up methodology for developing objective functions that, when maximized, lead to the estimation of the posterior probability. 
The bottom-up approach reverses the top-down procedure typical of $f$-divergence formulations \cite{nowozin2016f, wei2020optimizing}. 
The main advantage of this new method is that it guides the design of the objective function by starting with the imposition of the optimal convergence condition of the discriminator's output $D^{\diamond}(\mathbf{x}, \mathbf{y})$ (see Appendix \ref{subsec:appendix_insights_bottom_up}). 
Theorem \ref{theorem:Bottom-up} presents the class of objective functions that, when maximized, leads to the bottom-up posterior estimator.
\begin{theorem}  
\label{theorem:Bottom-up}
Let $X$ and $Y$ be the random vectors with pdfs $p_X(\mathbf{x})$ and $p_Y(\mathbf{y})$, respectively. Assume $\mathbf{y} = H(\mathbf{x})$, where $H(\cdot)$ is a stochastic function, then $p_{XY}(\mathbf{x}, \mathbf{y})$ is the joint density. Let $\mathcal{T}_x$ and $\mathcal{T}_y$ be the support of $p_X(\mathbf{x})$ and $p_Y(\mathbf{y})$, respectively.
Let the discriminator $D(\mathbf{x}, \mathbf{y})$ be a scalar function of $\mathbf{x}$ and $\mathbf{y}$. Let $k(\cdot)$ be any deterministic and invertible function.
Then, the posterior density is estimated as
\begin{equation}
\label{eq:bottom_up_posterior_estimator}
    \hat{p}_{X|Y}(\mathbf{x}|\mathbf{y}) = k^{-1}(D^{\diamond}(\mathbf{x}, \mathbf{y})) ,
\end{equation}
where $D^{\diamond}(\mathbf{x},\mathbf{y})$ is the optimal discriminator obtained by maximizing
\begin{equation}
\label{eq:value_function_general_bottom_up}
    \mathcal{J}(D) = \int_{\mathcal{T}_x} \int_{\mathcal{T}_y}  \tilde{\mathcal{J}}(D) d\mathbf{x} d\mathbf{y},
\end{equation}
for all concave functions $ \tilde{\mathcal{J}}(D)$ such that their first derivative is 
\begin{align}
\label{eq:g(D)}
    \frac{\partial \tilde{\mathcal{J}}(D)}{\partial D} &= \Bigl( D(\mathbf{x}, \mathbf{y}) - k(p_{X|Y}(\mathbf{x}|\mathbf{y}))\Bigr) g_1(D,k) \notag \\
    & \triangleq g(D, k)
\end{align}
with $g_1(D,k) \neq 0$ deterministic and $\frac{\partial g(D,k)}{ \partial D} \leq 0$.
\end{theorem}

The proposed estimator leverages a discriminative formulation to estimate the density ratio corresponding to the posterior probability. Theorem \ref{theorem:Bottom-up} comprises two DOFs in the objective function design. 
The former is the choice of $k(\cdot)$, since the discriminator estimates an invertible transformation of the posterior density. 
We noticed that the classifier performs better when $k(\cdot)$ is chosen to resemble particular activation functions. 
The second DOF, represented by $g_1(\cdot)$, is a rearrangement term that modifies the result of the integration of $g(\cdot)$ with respect to $D$. The exploitation of such a DOF allows to attain different objective functions even when imposing the same optimal discriminator's output. Theorems \ref{theorem:top_down} and \ref{theorem:Bottom-up} have different advantages. The benefit of Theorem \ref{theorem:top_down} is its simple applicability, that must, however, begin with the selection of the generator function $f$. Theorem \ref{theorem:Bottom-up} has the advantage of guiding the objective function design without relying on existing $f$-divergences. \\
Corollary \ref{corollary:bottom-up-top-down-value-functions} in Appendix \ref{subsec:appendix_corollary_bottom_up_top_down_value_functions} exploits the DOFs in the choice of $k(\cdot)$ and $g_1(\cdot)$ to obtain the objective functions in Tab. \ref{tab:value functions}, which were previously derived from Theorem \ref{theorem:top_down}.
An interesting result highlighted by Corollary \ref{corollary:bottom-up-top-down-value-functions} is that all the objective functions corresponding to the well-known $f$-divergences in Tab. \ref{tab:f_divergences_table_unsupervised} use the same class of functions $g_1(\cdot) \propto 1/D^{\alpha}\cdot 1/(1-D)^{\beta}$ expressed in \eqref{eq:g_1(D)} in Appendix \ref{subsec:appendix_corollary_bottom_up_top_down_value_functions}. 
This property, highlighted by the class of bottom-up estimators, shows us that $\mathcal{J}_{GAN}(D)$ is the only one using $\beta \neq 0$. 
Thanks to the observation of this peculiarity and driven by curiosity, we develop a new objective function starting from $k_{GAN}(\cdot)$ and $g_{1,GAN}(\cdot)$ but imposing $\beta=0$.

\section{Shifted Log Objective Function and Divergence}
\label{sec:New_f}

In this section, we present a new objective function for classification problems, that we design by using Theorem \ref{theorem:Bottom-up}. Then, we prove that such an objective function corresponds to the variational representation of a novel $f$-divergence, called \textit{shifted log} (SL). We will demonstrate in Section \ref{sec:Results} that such a new objective function achieves the best performance in almost all the classification scenarios discussed.  

\begin{theorem}
\label{theorem:new_f}
Let $X$ and $Y$ be two random vectors with pdfs $p_X(\mathbf{x})$ and $p_Y(\mathbf{y})$, respectively. Assume $Y = H(X)$, with $H(\cdot)$ stochastic function, then let $p_{XY}(\mathbf{x}, \mathbf{y})$ be the joint density. Let $\mathcal{T}_x$ be the support of $X$. Let $p_U(\mathbf{x})$ be a uniform pdf having the same support $\mathcal{T}_x$.  
The maximization of the objective function 
\begin{align}
\label{eq:fNOME_cost_fcn}
    &\mathcal{J}_{SL}(D) = - \E_{(\mathbf{x},\mathbf{y}) \sim p_{XY}(\mathbf{x}, \mathbf{y})} \Bigl[ D(\mathbf{x}, \mathbf{y}) \Bigr] \notag \\
    & + \E_{(\mathbf{x},\mathbf{y}) \sim p_U(\mathbf{x})p_Y(\mathbf{y})} \Bigl[ |\mathcal{T}_x| \Bigl( \log(D(\mathbf{x}, \mathbf{y})) -  D(\mathbf{x}, \mathbf{y}) \Bigr) \Bigr],
\end{align}
leads to the optimal discriminator output 
\begin{equation}
\label{eq:fNAME_D_opt}
    D^{\diamond}(\mathbf{x}, \mathbf{y}) = \argmax_{D} \mathcal{J}_{SL}(D) = \frac{1}{1 + p_{X|Y}(\mathbf{x}, \mathbf{y})} ,
\end{equation}
and the posterior density estimate is computed as
\begin{equation}
\label{eq:posterior_estimator_sl}
    \hat{p}_{X|Y}(\mathbf{x}|\mathbf{y}) = \frac{1 - D^{\diamond}(\mathbf{x},\mathbf{y})}{D^{\diamond}(\mathbf{x},\mathbf{y})} .
\end{equation}
\end{theorem}

Corollary \ref{corollary:f_SL} states that $\mathcal{J}_{SL}(D)$ in \eqref{eq:fNOME_cost_fcn} can be obtained from Theorem \ref{theorem:top_down} by using a new $f$-divergence referred to as shifted log.

\begin{corollary}
\label{corollary:f_SL}
    Define the generator function of the shifted log divergence
\begin{equation}
\label{eq:fNOME_f}
    f_{u, SL}(u) = - |\mathcal{T}_x| \log(u + |\mathcal{T}_x|) + K ,
\end{equation}
where $K=|\mathcal{T}_x| \log(1 + |\mathcal{T}_x|)$ is constant. Then, $\mathcal{J}_{SL}(D)$ in \eqref{eq:fNOME_cost_fcn} is the variational representation of $D_{f_{u,SL}}(p_{XY}||p_Up_Y)$.
\end{corollary}


\subsection{Remarks on the New Objective Function and $f$-Divergence}
\begin{figure*}[t]
    \centering
    \subfloat[\centering Unsupervised architecture \label{fig:unsupervisedArchitecture}]{{\includegraphics[width=7cm]{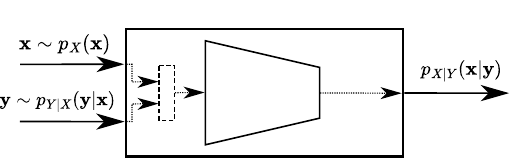} }}%
    \qquad
    \subfloat[\centering Supervised architecture \label{fig:supervisedArchitecture}]{{\includegraphics[width=7cm]{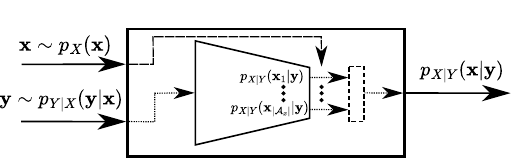} }}%
    \caption{Diagrams of unsupervised and supervised architectures. The thick rectangle delineates the \textit{discriminator} in Fig. \ref{fig:system-model}. The trapezoidal shape represents the neural network architecture.}%
    \label{fig:architectures}
\end{figure*}
Since the proposed objective function in \eqref{eq:fNOME_cost_fcn} is the variational representation of an $f$-divergence (for Corollary \ref{corollary:f_SL}), Lemma \ref{lemma:convergence} ensures the convergence property of its estimate to the true posterior probability in the nonparametric limit.

The supervised version of the SL divergence is 
\begin{equation}
\label{eq:general_fSL}
    f_{SL}(u) = -\log(u + 1) + \log(2) ,
\end{equation}
which is obtained by substituting $|\mathcal{T}_x|=1$ (see Appendix \ref{subsec:Appendix_value_functions}).
Although $f_{SL}$ is obtained in the context of posterior estimation problems, the proposed $f$-divergence can be applied to a broader variety of tasks.
Since $D_{SL}$ is upper-bounded (see Corollary \ref{corollary:upper_lower_bound_newF} in Appendix \ref{subsec:appendix_corollary_upper_lower_bound_SL}), it is a suitable generator function when the optimization problem requires to maximize the variational representation of the $f$-divergence (i.e., a max-max game). For instance, for classification tasks as in \cite{wei2020optimizing} or for representation learning applications as in \cite{hjelm2018learning}. 

\subsection{Comparison Between SL and GAN Divergences}
\label{subsec:Theoretical_comparison_GAN_newf}
    The GAN divergence is known to be highly-performing in a wide variety of tasks \cite{nowozin2016f, hjelm2018learning}. Moreover, the objective functions corresponding to SL and GAN divergences can be obtained from Theorem \ref{theorem:Bottom-up} by choosing the same $k(\cdot)$, but different $g_1(\cdot)$. 
    Corollary \ref{corollary:unsupervised_GAN_SL} compares the concavity of the two objective functions in the neighborhood of the global optimum. 

    \stepcounter{theorem}
     \begin{corollary}
    \label{corollary:unsupervised_GAN_SL}
        Let $\mathcal{J}_{SL}(D)$ be defined as in \eqref{eq:fNOME_cost_fcn}. Let $\mathcal{J}_{GAN}(D)$ be defined as in Tab. \ref{tab:value functions}. Let $D^{\diamond}_N$ be the discriminator output in a neighborhood of $D^{\diamond}$ where $\mathcal{J}_{SL}(D^{\diamond}_N)$ and $\mathcal{J}_{GAN}(D^{\diamond}_N)$ are concave.
        Then,
        \begin{equation}
        \label{eq:unsupervised_GAN_SL}
            \Bigg| \frac{\partial \mathcal{J}_{GAN}(D^{\diamond}_N)}{\partial D } \Bigg| \geq \Bigg| \frac{\partial \mathcal{J}_{SL}(D^{\diamond}_N)}{\partial D } \Bigg| .
        \end{equation}
    \end{corollary}
    Although the shape of the loss landscape depends on many factors (e.g., the batch size \cite{keskar2016large, chaudhari2019entropy}), sharper maxima attain larger test error \cite{li2018visualizing}. 
    The steepness of the concavity of the objective function in the neighborhood of the global optimum provides insights on the basin of attraction of the point of maximum. Intuitively, a flatter landscape corresponds to a larger basin of attraction of the global optimum, rendering training with the SL divergence better than with the GAN divergence. The results in Fig. \ref{fig:speed_accuracy} and Tab. \ref{tab:speed_accuracy} in Appendix \ref{subsec:appendix_numerical_results} validate the findings of Corollary \ref{corollary:unsupervised_GAN_SL}.
    

\section{Discriminator Architecture}
\label{sec:Architecture}
 
In this section, we discuss the appropriate modifications to the discriminator's architecture to suit our estimators to the classification scenario, where the number of classes is finite. 
The architecture type used differs depending on the alphabet of $X$, referred to as $\mathcal{A}_x$. When $X$ is continuous, we use a structure referred to as \textit{unsupervised architecture}. Differently, when $X$ is discrete, we use a structure referred to as \textit{supervised architecture}.

\subsection{Unsupervised Architecture}

In this setting, the samples $\mathbf{x}$ and $\mathbf{y}$ drawn from the empirical probability distributions $p_{XY}(\mathbf{x}, \mathbf{y})$ and $p_U(\mathbf{x})p_Y(\mathbf{y})$ are concatenated and fed into the discriminator. 
The discriminator output is a scalar, since it is the posterior density function estimate corresponding to the pair $(\mathbf{x}, \mathbf{y})$ given as input.
The discriminator architecture is represented in Fig. \ref{fig:architectures}(a), where the concatenation between the $\mathbf{x}$ and $\mathbf{y}$ realizations is identified by a dashed rectangle. 

\subsection{Supervised Architecture}
The supervised architecture (Fig. \ref{fig:architectures}(b)) introduces in the problem's formulation the constraint that $\mathcal{A}_x$ is a set containing a finite number of elements $\mathcal{A}_x = \left\{ \mathbf{x}_1, \dots, \mathbf{x}_m \right\}$.
This constraint is embedded in the architecture (highlighted by a dashed arrow in Fig. \ref{fig:architectures}(b)), so that the output layer contains one neuron for each sample in $\mathcal{A}_x$.
With this modification, the \textit{i}-th output neuron returns $p_{X|Y}(\mathbf{x}_i | \mathbf{y})$.
Accordingly, the input layer is fed with only the realizations $\mathbf{y}$. 
Notation wise, the discriminator's output is referred to as $\textbf{D}(\mathbf{y})$.
Theorem \ref{theorem:supervised_value_function} shows the modification to the formulation of the objective functions in \eqref{eq:top_down_variational_representation_posterior} when using a supervised architecture. First, we define the notation useful for the theorem statement: $\textbf{D}(\mathbf{y}) = [D(\mathbf{x}_1, \mathbf{y}), \dots, D(\mathbf{x}_m, \mathbf{y})]$ and $\textbf{1}_m(\mathbf{x}_i)=[0, \dots, 0, \underbrace{1}_{i^{th} \>\> pos.}, 0, \dots, 0]^T$.

\begin{theorem}
\label{theorem:supervised_value_function}
    Let $p_X(\mathbf{x})$ and $p_Y(\mathbf{y})$ be pdfs describing the input and output of a stochastic function $H(\cdot)$, respectively. Let $p_X(\mathbf{x}) \triangleq \sum_{i = 1}^m P_X(\mathbf{x}_i) \delta (\mathbf{x} - \mathbf{x}_i)$, where $P_X(\cdot)$ is the probability mass function of $X$. Let $\mathcal{T}_x$ be the support of $p_X(\mathbf{x})$ and $|\mathcal{T}_x|$ its Lebesgue measure.
    Let $p_U(\mathbf{x})$ be the uniform discrete pdf over $\mathcal{T}_x$. Let the discriminator be characterized by a supervised architecture. Then, the objective function in \eqref{eq:top_down_variational_representation_posterior} becomes
    \begin{align}
    \label{eq:supervised_general_value_function}
        \mathcal{J}(D) &= \E_{\mathbf{x} \sim p_X(\mathbf{x})}\Biggl[ \E_{\mathbf{y} \sim p_{Y|X}(\mathbf{y}|\mathbf{x})}\Bigl[ r(\textbf{D}(\mathbf{y})) \textbf{1}_m(\mathbf{x}) \Bigr] \Biggr] \notag \\
        & - \E_{\mathbf{y} \sim p_Y(\mathbf{y})}\Biggl[ \sum_{i=1}^m f^{*} \left(r(D(\mathbf{x}_i,\mathbf{y})) \right) \Biggr] ,
    \end{align}
    where $D(\mathbf{x}_i,\mathbf{y})$ is the i-th component of $\textbf{D}(\mathbf{y})$ and $\textbf{T}(\mathbf{y}) = r(\textbf{D}(\mathbf{y}))$.
\end{theorem}
The supervised versions of the objective functions utilized in this paper are listed in Section \ref{subsec:Appendix_value_functions} of the Appendix.\\
For classification problems, the objective function obtained by substituting the KL divergence in Theorem \ref{theorem:supervised_value_function} (see \eqref{eq:kl_value_function_supervised} in Appendix \ref{subsubsec:appendix_KL_obj_fcn}) is exactly the cross-entropy loss. When the Softmax function is applied to the discriminator output (because $D^\diamond(\mathbf{x}, \mathbf{y})=p_{X|Y}(\mathbf{x}|\mathbf{y})$ is a discrete pdf), the expectation over $p_Y(\mathbf{y})$ becomes a constant always equal to $1$. Maximizing \eqref{eq:kl_value_function_supervised} is equivalent to minimizing the negative of \eqref{eq:kl_value_function_supervised}, which is precisely the minimization of the CE. 

\section{Results}
\label{sec:Results}
\setcounter{table}{3}

\begin{table*}
\caption{Classification accuracy on MNIST (M), Fashion MNIST (FM), CIFAR10 (C10), and CIFAR100 (C100). The MobileNetV2 is referred to as MobileNet.} 
  \begin{center}
  \begin{small}
  \begin{sc}
    \begin{tabular}{ c c c c c c c c } 
     \hline
     Dataset & Model & CE & RKL & HD & GAN & P & SL \\
     \hline
     M & Tiny & $\textbf{99.08} \pm 0.06$ & $96.05 \pm 0.25$ & $98.68 \pm 0.05$ & $\textbf{99.08} \pm 0.07$ & $98.89 \pm 0.08$ & $99.03 \pm 0.04$ \\
     \hline
     FM & Tiny & $91.64 \pm 0.09$ & $82.63 \pm 1.78$ & $90.75 \pm 0.13$ & $91.63 \pm 0.10$ & $89.86 \pm 0.67$ & $\textbf{91.83} \pm 0.02$ \\
     \hline
      \multirow{5}{*}{C10} & Tiny & $70.13 \pm 0.05$ & $63.59 \pm 0.34$ & $69.38 \pm 0.28$ & $69.98 \pm 0.15$ & $59.62 \pm 0.45$ & $\textbf{70.87} \pm 0.26$  \\
      & VGG & $93.69 \pm 0.03$ & $84.24 \pm 2.21$ & $93.51 \pm 0.06$ & $93.75 \pm 0.04$ & $84.79 \pm 0.21$ & $\textbf{93.93} \pm 0.08$\\
      & DLA & $95.04 \pm 0.02$ & $90.83 \pm 0.10$ & $94.56 \pm 0.11$ & $95.04 \pm 0.13$ & $91.61 \pm 0.21$ & $\textbf{95.31} \pm 0.09$\\
      & ResNet & $95.39 \pm 0.04$ & $92.88 \pm 0.26$ & $95.15 \pm 0.08$ & $95.24 \pm 0.06$ & $93.78 \pm 0.21$ & $\textbf{95.43} \pm 0.04$\\
      & MobileNet & $92.59 \pm 0.13$ & $83.97 \pm 0.21$ & $91.95 \pm 0.33$ & $92.37 \pm 0.14$ & $84.30 \pm 0.32$ & $\textbf{93.89} \pm 0.15$ \\
      \hline
      \multirow{4}{*}{C100} & VGG & $72.73 \pm 0.30$ & $45.80 \pm 2.86$ & $73.51 \pm 0.03$ & $68.88 \pm 0.20$ & $37.19 \pm 0.66$ & $\textbf{73.61} \pm 0.05$\\
      & DLA & $76.29 \pm 0.43$ & $68.86 \pm 1.17$ & $78.63 \pm 0.14$ & $77.34 \pm 0.22$ & $57.97 \pm 0.07$ & $\textbf{78.65} \pm 0.01$\\
      & ResNet & $\textbf{78.29} \pm 0.18$ & $ 70.68 \pm 0.44$ & $77.59 \pm 0.06$ & $77.43 \pm 0.08$ & $61.12 \pm 0.23$ & $78.03 \pm 0.04$\\
      & MobileNet & $72.61 \pm 0.08$ & $53.17 \pm 0.35$ & $73.00 \pm 0.30$ & $65.66 \pm 0.46$ & $46.00 \pm 0.37$ & $\textbf{74.78} \pm 0.23$ \\
     \hline
    \end{tabular}
    \end{sc}
    \end{small}
    \end{center}
    \vskip -0.1in
    \label{tab:classification_images}
\end{table*}

In this section, we report several numerical results to assess the validity of the methods proposed to estimate the posterior probability density and enable the classification task. The considered scenarios are: classification for image datasets; signal decoding in telecommunications engineering cast into a classification task; posterior probability estimation when $p_X(\mathbf{x})$ is continuous.
The results demonstrate that different $f$-divergences attain different performance. 
We show that the KL divergence (thus the CE) is not necessarily the best choice for classification tasks, and more in general for probability estimation problems. 
We demonstrate that the SL divergence achieves the best performance in almost all the tested contexts. When referring to the performance of any $f$-divergence, we will implicitly imply the performance achieved using the objective function derived using such an $f$-divergence.\\
The first two scenarios are classification problems, therefore we use the supervised formulation of the discriminator architecture. The third scenario comprises two toy cases, where we show that the unsupervised formulation of the proposed estimators works for continuous random vectors $X$. 
Before discussing the numerical results, we briefly describe the details of the code implementation\footnote{Our implementation can be found at \url{https://github.com/tonellolab/discriminative-classification-fDiv}}.

\subsection{Implementation Details}
\label{subsec:implementation details}

\textbf{Supervised Architecture}: For the first scenario (Section \ref{subsubsec:numerical_results_images}), we use convolutional neural networks. 
When referring to \textit{tiny} network, we use a discriminator comprising a small set of convolutional layers (less than 4) followed by a feedforward fully connected part. Besides the tiny network, in the first scenario we utilize various deep network architectures: VGG \cite{DBLP:journals/corr/SimonyanZ14a}, ResNet \cite{resnet}, DLA \cite{yu2018deep}, and MobileNetV2 \cite{sandler2018mobilenetv2}.
The discriminator hyper-parameters slightly vary depending on the dataset tested. The network parameters are updated by using SGD with momentum. The activation function of the last layer depends on the objective function optimized during the training phase.\\
For the second scenario (Section \ref{subsubsec:numerical_results_decoding}), we use fully connected feedforward neural networks.
The architecture used for the decoding scenario comprises two hidden layers with $100$ neurons each. The network weights are updated by using the Adam optimizer \cite{DBLP:journals/corr/KingmaB14}. The LeakyReLU activation function is utilized in all the layers except the last one, where the activation function is chosen based on the objective function. In some cases, the Dropout technique \cite{JMLR:v15:srivastava14a} helps the convergence of the training process. 

\textbf{Unsupervised Architecture}: The discriminator architecture utilized for the unsupervised tasks comprises two hidden layers with $100$ neurons each and the LeakyReLU activation function. 
The activation function of the output layer depends on the objective function used during training. The network weights and biases are updated by using the Adam optimizer. Dropout is used during training.

\subsection{Image Datasets Classification}
\label{subsubsec:numerical_results_images}
The first scenario tackled is the classification of image datasets. The objective functions performance is tested for the MNIST \cite{lecun1998gradient}, Fashion MNIST \cite{xiao2017fashion}, CIFAR10, and CIFAR100 \cite{krizhevsky2009learning} datasets. A more detailed description of the datasets can be found in Appendix \ref{subsec:appendix_numerical_results}. 
We compare the classification accuracy of the supervised versions of the objective functions in Tab. \ref{tab:value functions} and in \eqref{eq:fNOME_cost_fcn}, which are all reported in Appendix \ref{subsec:Appendix_value_functions}. To improve the training procedure, we apply data augmentation on the CIFAR datasets by randomly cropping and flipping the images. The learning rate is initially set to $0.1$ and then we use a cosine annealing scheduler \cite{loshchilov2016sgdr} to modify its value during the $200$ epochs of training.  
\begin{figure}[ht]
	\centerline{\includegraphics[width=0.9\columnwidth]{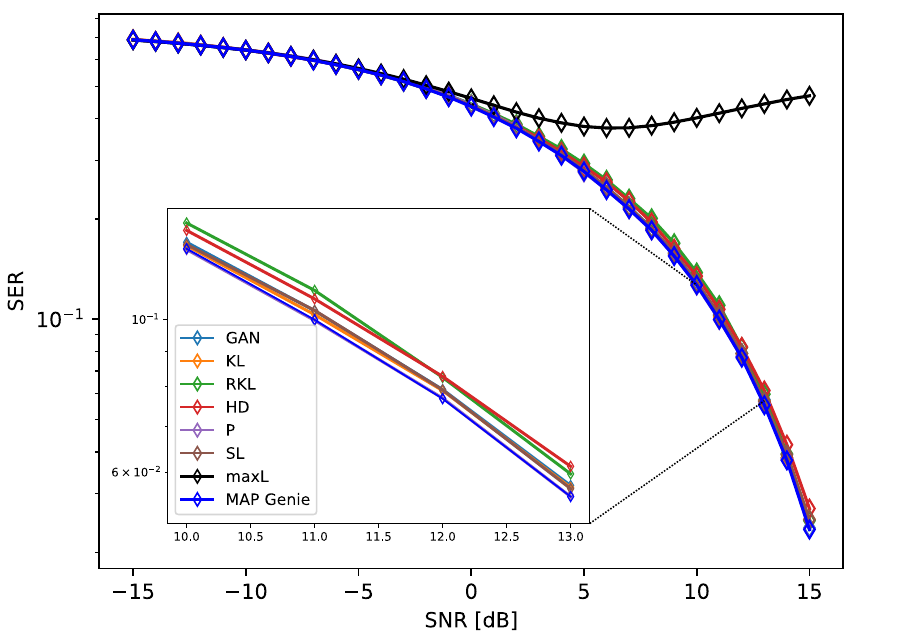}}
	\caption{SER achieved by using a 4-PAM modulation over a nonlinear communication channel.}
	\label{fig:4-PAM}
\end{figure}
\begin{figure*}[ht]
	\centerline{\includegraphics[width=\textwidth]{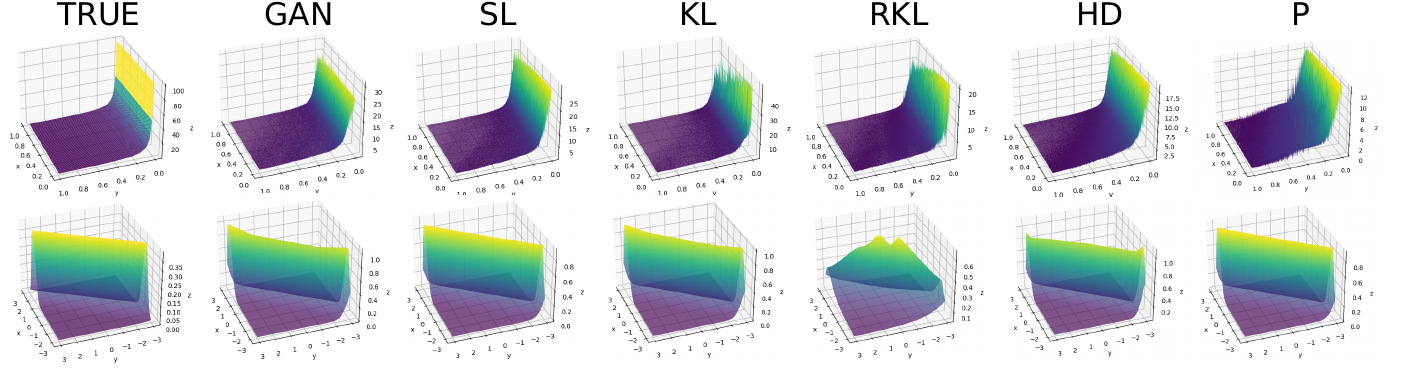}}
	\caption{Continuous posterior density estimation for various $f$-divergences. The results of the Exponential task is represented in the upper row, while the outcomes of the Gaussian task are depicted in the lower row. The true posterior density is the first plot of each row.}
	\label{fig:exponential_gaussian}
\end{figure*}
We compute the mean accuracy and its standard deviation for each dataset and $f$-divergence by running the code over multiple random seeds. 
The classification accuracy results displayed in Tab. \ref{tab:classification_images} (and in the extended version comprising more network architectures in Tab. \ref{tab:classification_images_appendix} in Appendix \ref{subsec:appendix_numerical_results}) confirm that each generator function has a different impact on the neural network training, as also shown in \cite{nowozin2016f, wei2020optimizing}. 
The SL divergence attains the highest classification accuracy for almost all the architectures tested, showing its effectiveness compared to the other divergences and its stable behavior over different datasets and architectures. In particular, the SL divergence attains better performance than the CE, which corresponds to the state-of-the-art approach for image classification tasks. Furthermore, in the few scenarios where the SL divergence does not achieve the best performance, it achieves the second best (only in one case the third best, in Table \ref{tab:classification_images_appendix} in Appendix \ref{subsec:appendix_numerical_results}) performance, with an accuracy close to the optimal one (see Tables \ref{tab:classification_images}, \ref{tab:classification_images_appendix}). The CE is, on average, the second-best objective function.
Conversely, the Pearson $\chi^2$ performs the worst in almost all the scenarios. Moreover, except from the SL divergence and the CE, the performance of the other $f$-divergences is more architecture dependent. For instance, as reported in Tab. \ref{tab:classification_images_appendix} in Appendix \ref{subsec:appendix_numerical_results}, the PreActResNet \cite{he2016identity} attains the highest accuracy when it is trained with the GAN-based objective function. 
The choice of the architecture often depends on the goal of the classification algorithm. For embedded systems, light architectures are used. Therefore, the MobileNetV2 is an option \cite{chiu2020mobilenet}. In such a case, the SL divergence obtains an accuracy $1.5/2\%$ higher than the CE, which makes the SL the preferred choice for the network's training.
Additional numerical results on other network architectures and on the speed of convergence of the training phase are reported in Appendix \ref{subsec:appendix_numerical_results} for space limitations. The speed of convergence analysis (Fig. \ref{fig:speed_accuracy}) demonstrates that the training with the SL divergence leads to a faster convergence to the optimum discriminator w.r.t. the GAN divergence, as stated in Corollary \ref{corollary:unsupervised_GAN_SL}. 

\subsection{Signal Decoding}
\label{subsubsec:numerical_results_decoding}
The second scenario is the decoding problem. Decoding a sequence of received bits is crucial in a telecommunication system. In some cases, when the communication channel is known, the optimal decoding technique is also known \cite{proakis2007fundamentals}. However, the communication channel is generally unknown, and DL-based techniques can be used to learn it \cite{Mehran2019Deep}. By knowing that the optimal decoding criterion is the posterior probability maximization, we demonstrate that the proposed MAP approach solves the decoding problem and that the supervised version of $\mathcal{J}_{SL}(D)$ (see \eqref{eq:supervised_SL_value_function} in Appendix \ref{subsubsec:appendix_SL_obj_fcn}) achieves optimal performance. 
We consider a 4-PAM (i.e., pulse amplitude modulation with four symbols) over a nonlinear channel with additive Gaussian noise. In particular, given the symbol at time instant $t$ (referred to as $x_t$), we obtain the channel output as $y_t = sgn(x_t)\sqrt{|x_t|} + n_t$, where $n_t$ is the Gaussian noise and $sgn(\cdot)$ is the sign function.
We show the symbol error rate (SER) behavior achieved by the proposed posterior estimators when varying the signal-to-noise ratio (SNR) in Fig. \ref{fig:4-PAM}. The estimators proposed are compared with the results of the max likelihood estimator (referred to as maxL) and the MAP Genie estimators. The MAP Genie estimator uses the knowledge of the channel nonlinearity to decode the received sequence of bits.
The proposed list of posterior probability estimators performs better than the max likelihood estimator, achieving accuracy close to the optimal MAP Genie estimator. In Appendix \ref{subsec:appendix_numerical_results}, other scenarios of decoding tasks are reported.  

\subsection{Continuous Posterior Estimation}
This section considers two toy examples for the case $|\Omega| \geq |\mathbb{R}|$. The comparison between the closed-form of the posterior distribution and the discriminator estimate is showed in Fig. \ref{fig:exponential_gaussian}. The discriminator prediction is obtained by training a tiny neural network with the objective functions reported in Tab. \ref{tab:value functions} and in \eqref{eq:fNOME_cost_fcn}. The closed-form posterior distribution (referred to as \textit{true} in Fig. \ref{fig:exponential_gaussian}) is the first element of each row.

\textbf{Exponential task}. In the first toy example, we define the model $Y= X+N$, where $X$ and $N$ are independent Exponential random variables. Therefore, $Y$ is a Gamma distribution \cite{durrett2019probability}. 
The closed-form posterior probability can be expressed as in \eqref{eq:posterior_closed_form_exponential} (see Appendix \ref{subsec:appendix closed form exponential} for the formula and proof). Similarly to the previous numerical results, different $f$-divergences lead to diverse estimates of the posterior density. The objective functions corresponding to the GAN, SL, and HD divergences attain better estimates w.r.t. the others. For a fixed $y$, in fact, the posterior density value is constant over $x$.

\textbf{Gaussian task}. In the second toy example, we consider the model $Y = X+N$, where $X$, $N$ are independent Gaussian random variables. Thus, $Y$ is a Gaussian distribution. The posterior density expression is reported in Appendix \ref{subsec:appendix closed form gaussian} (see \eqref{eq:posterior_closed_form_gaussian}). The objective functions corresponding to the SL, KL, and P divergences attain better estimates w.r.t. the others, since the estimate attains the desired saddle shape.


\section{Conclusions}
\label{sec:Conclusions}

In this paper, we proposed a new MAP perspective for supervised classification problems. 
We have proposed to use a discriminative formulation to express the posterior probability density, and we have derived two classes of estimators to estimate it.
From them, we extracted a list of posterior probability estimators and compared them with the notorious cross-entropy minimization approach. 
Numerical results on different scenarios demonstrate the effectiveness of the presented estimators and that the proposed SL divergence achieves the highest classification accuracy in almost all the scenarios.
Additionally, we show that the proposed posterior probability estimators work for the general case of continuous a priori probabilities, for which we design a specific neural network architecture.

\section*{Acknowledgements}
We thank Nunzio Alexandro Letizia for the precious insights and the fruitful discussions. We also thank the anonymous ICML 2024
reviewers for their detailed and helpful feedback.

\section*{Impact Statement}
This paper presents work whose goal is to advance the field of Machine Learning. There are many potential societal consequences of our work, none of which we feel must be specifically highlighted here.


\bibliography{main_icml2024}
\bibliographystyle{icml2024}

\newpage
\appendix
\onecolumn
\section{Appendix: Objective Functions Used in the Experiments}
\label{subsec:Appendix_value_functions}

\setcounter{table}{1}

\begin{table}
\caption{Unsupervised $f$-divergences table. The corresponding $f$-divergences are: Kullback-Leibler, Reverse Kullback-Leibler, squared Hellinger distance, GAN, and Pearson $\chi^2$.} 
\centering
\vskip 0.15in
  \begin{center}
  \begin{small}
  \begin{sc}
    \begin{tabular}{ c c c } 
     \hline
     Name & $f_{u}(u)$ & $f_{u}^{*}(t)$ \\
     \hline
      KL & $u \log\left(\frac{u}{|\mathcal{T}_x|}\right)$ & $|\mathcal{T}_x|\exp{(t-1)}$  \\
      RKL & $-|\mathcal{T}_x|\log(u)$ & $-|\mathcal{T}_x|(1 + \log(-t))$  \\
      HD & $(\sqrt{u} - \sqrt{|\mathcal{T}_x|})^2$ & $|\mathcal{T}_x|\frac{t}{1-t}$\\
      GAN & $u \log(u) - (u+|\mathcal{T}_x|)\log(u+|\mathcal{T}_x|)$ & $-|\mathcal{T}_x|\log(1-\exp{(t)})$ \\
      P & $\frac{1}{|\mathcal{T}_x|}(u-|\mathcal{T}_x|)^2$ & $|\mathcal{T}_x| \left( \frac{1}{4}t^2 + t \right) $ \\
     \hline
    \end{tabular}
    \end{sc}
    \end{small}
    \end{center}
    \vskip -0.1in
    \label{tab:f_divergences_table_unsupervised}
\end{table}

\begin{table}
\caption{$f$-divergences table. The corresponding $f$-divergences are: Kullback-Leibler, Reverse Kullback-Leibler, squared Hellinger distance, GAN, and Pearson $\chi^2$.} 
\centering
\vskip 0.15in
  \begin{center}
  \begin{small}
  \begin{sc}
    \begin{tabular}{ c c c } 
     \hline
     Name & $f(u)$ & $f^{*}(t)$ \\
     \hline
      KL & $u \log(u)$ & $\exp{(t-1)}$  \\
      RKL & $-\log(u)$ & $-1- \log(-t)$  \\
      HD & $(\sqrt{u} -1)^2$ & $\frac{t}{1-t}$\\
      GAN & $u \log(u) - (u+1)\log(u+1)$ & $-\log(1-\exp{(t)})$ \\
      P & $(u-1)^2$ & $\frac{1}{4}t^2 + t$ \\
     \hline
    \end{tabular}
    \end{sc}
    \end{small}
    \end{center}
    \vskip -0.1in
    \label{tab:f_divergences_table}
\end{table}

The unsupervised and supervised versions of the objective functions used to achieve the results showed in Section \ref{sec:Results} are reported in this section. The training part consists in the alternation of two phases. In the former phase, the network is fed with $N$ realizations of the joint distribution $p_{XY}(\mathbf{x}, \mathbf{y})$ to compute the first term of the objective function. In the latter phase, the type of architecture defines the procedure to compute the second term of the objective function. The unsupervised architecture is fed with $N$ samples drawn from $p_U(\mathbf{x})p_Y(\mathbf{y})$. The supervised architecture is fed with $N$ samples drawn from $p_Y(\mathbf{y})$ (see Theorem \ref{theorem:supervised_value_function}).\\ 
During the test part, the network is fed with the samples drawn from the joint distribution ($(\mathbf{x}, \mathbf{y}) \sim p_{XY}(\mathbf{x}, \mathbf{y})$), and the posterior probability density estimate is obtained as in \eqref{eq:top_down_posterior_estimator} or \eqref{eq:bottom_up_posterior_estimator}.\\ 
We report the objective functions derived from well-known $f$-divergences. Therefore, we first report the notorious $f$-divergences in Table \ref{tab:f_divergences_table}, and their unsupervised version in Table \ref{tab:f_divergences_table_unsupervised}. These tables do not contain the constant terms that render $f_u(1)=0$ or $f(1)=0$, as their presence do not affect the optimization of the derived objective functions (see the proof of Lemma \ref{lemma:f_u_valid_divergence}). 
Let $f_{u}$ and $f_{u}^{*}$ the unsupervised generator function and its Fenchel conjugate in Table \ref{tab:f_divergences_table_unsupervised}, respectively. Differently, $f$ and $f^{*}$ are the supervised generator function and its Fenchel conjugate in Table \ref{tab:f_divergences_table}, respectively.
Then, 
\begin{equation}
\label{eq:unsupervised_f_def}
    f_{u}^{*}(t) = |\mathcal{T}_x| f^{*}(t),
\end{equation}
where $|\mathcal{T}_x|$ is the Lebesgue measure of the support of $p_X(\mathbf{x})$.
Accordingly, 
\begin{equation}
    f_{u}(u) = \sup_{t \in dom_{f_{u}^{*}}} \left\{ ut - f_{u}^{*}(t) \right\} .
\end{equation}
Vice versa, the supervised version of $f_u$ and $f_u^*$ can be easily attained by substituting $|\mathcal{T}_x|=1$ in the unsupervised formulation. \\
The usage of \eqref{eq:unsupervised_f_def} in Theorem \ref{theorem:top_down} is needed because $|\mathcal{T}_x|$ counterbalances the effect of $p_U(\mathbf{x})$ to obtain the pdf ratio equivalent to the posterior density $p_{X|Y}(\mathbf{x}| \mathbf{y}) = p_{XY}(\mathbf{x}, \mathbf{y})/p_Y(\mathbf{y})$. In fact, $|\mathcal{T}_x|$ is included in the expectation term computed over $p_U(\mathbf{x})$.
In Lemma \ref{lemma:f_u_valid_divergence}, we prove that $f_u$ is the generator function of a valid $f$-divergence.

\textbf{Note}: The notations \textit{supervised} and \textit{unsupervised} refer to the discriminator architecture. The former one is the typical supervised classification architecture, while the latter one does not use $X$ as the set of labels, but as an additional input. 
The terminology we use to distinguish between the two different versions of $f$-divergences ($f$ and $f_u$) is a consequence of the architecture's notation. The unsupervised architecture, in fact, is trained using objective functions defined with $f_u$ (see Theorem \ref{theorem:top_down} and Corollary \ref{corollary:f_SL}), while the supervised architecture is trained by using objective functions based on $f$ (see Theorem \ref{theorem:supervised_value_function}). 

We recall the notation used in the objective functions: $\textbf{D}(\mathbf{y}) = [D(\mathbf{x}_1, \mathbf{y}), \dots, D(\mathbf{x}_m, \mathbf{y})]$ is a row vector, $\textbf{1}_m = [1, \dots, 1]^T$ is a column vector and $\textbf{1}_m(\mathbf{x}_i)=[0, \dots, 0, \underbrace{1}_{i^{th} \>\> pos.}, 0, \dots, 0]^T$ is a column vector.

\subsection{KL-Based Objective Functions}
\label{subsubsec:appendix_KL_obj_fcn}
\textbf{Kullback-Leibler Divergence}: The variational representation of the KL divergence is achieved substituting $f_{u,KL}^{*}$ listed in Table \ref{tab:f_divergences_table_unsupervised} in \eqref{eq:top_down_variational_representation_posterior}. This leads to the unsupervised objective function
\begin{align}
\label{eq:kl_value_function}
    \mathcal{J}_{KL}(D) &= \E_{(\mathbf{x}, \mathbf{y}) \sim p_{XY}(\mathbf{x}, \mathbf{y})} \Bigl[ \log(D(\mathbf{x}, \mathbf{y})) \Bigr] - \E_{(\mathbf{x}, \mathbf{y}) \sim p_U(\mathbf{x})p_Y(\mathbf{y})} \Bigl[ |\mathcal{T}_x| D(\mathbf{x}, \mathbf{y}) \Bigr] + 1 ,
\end{align}
where $T^{\diamond}(\mathbf{x}, \mathbf{y}) = \log(\frac{p_{XY}(\mathbf{x}, \mathbf{y})}{|\mathcal{T}_x|p_U(\mathbf{x})p_Y(\mathbf{y})}) + 1$, $p_U(\mathbf{x}) = \frac{1}{|\mathcal{T}_x|}$, and $D^{\diamond}(\mathbf{x}, \mathbf{y})= \frac{p_{XY}(\mathbf{x}, \mathbf{y})}{p_Y(\mathbf{y})}$. 
The supervised version of the objective function in \eqref{eq:kl_value_function} is attained by using Theorem \ref{theorem:supervised_value_function}:
\begin{align}
\label{eq:kl_value_function_supervised}
    \mathcal{J}_{KL}(D) &= \E_{\mathbf{x} \sim p_{X}(\mathbf{x})} \Biggl[ \E_{\mathbf{y} \sim p_{Y|X}(\mathbf{y}| \mathbf{x})} \Bigl[ \log(\textbf{D}(\mathbf{y})) \textbf{1}_m(\mathbf{x})  \Bigr] \Biggr] - \E_{ \mathbf{y} \sim p_Y(\mathbf{y})} \Bigl[ \textbf{D}(\mathbf{y}) \textbf{1}_m \Bigr] .
\end{align}
When the last layer of the supervised discriminator utilizes the softmax activation function (i.e., when the output is normalized to be a discrete probability density function), then the second term in \eqref{eq:kl_value_function_supervised} is always equal to 1. 
Thus, the maximization of \eqref{eq:kl_value_function_supervised} exactly corresponds to the minimization of the KL divergence in \eqref{eq:DKL definition}, and therefore to the minimization of the cross-entropy.\\
Interestingly, the more general formulation in \eqref{eq:kl_value_function_supervised} allows the usage of different activation functions in the last layer, with the only requirement that the discriminator's output is constrained to assume positive values (e.g., softplus). \\

\subsection{RKL-Based Objective Functions}
\label{subsubsec:appendix_RKL_obj_fcn}
\textbf{Reverse Kullback-Leibler Divergence}: Theorem \ref{theorem:top_down} leads to the variational representation of the reverse KL divergence, when substituting $f_{u,RKL}^{*}$ listed in Table \ref{tab:f_divergences_table_unsupervised} in \eqref{eq:top_down_variational_representation_posterior} 
\begin{align}
\label{eq:reverseKL_value_function}
    \mathcal{J}_{RKL}(D) &= \E_{(\mathbf{x}, \mathbf{y}) \sim p_{XY}(\mathbf{x}, \mathbf{y})} \Bigl[ - D(\mathbf{x}, \mathbf{y}) \Bigr] + \E_{(\mathbf{x}, \mathbf{y}) \sim p_U(\mathbf{x})p_Y(\mathbf{y})} \Bigl[ |\mathcal{T}_x| \log(D(\mathbf{x}, \mathbf{y})) \Bigr] ,
\end{align}
where $T^{\diamond}(\mathbf{x}, \mathbf{y}) = -\frac{|\mathcal{T}_x|p_U(\mathbf{x})p_Y(\mathbf{y})}{p_{XY}(\mathbf{x}, \mathbf{y})}$ and $D^{\diamond}(\mathbf{x}, \mathbf{y}) = \frac{p_Y(\mathbf{y})}{p_{XY}(\mathbf{x}, \mathbf{y})}$. 
The supervised version of the objective function in \eqref{eq:reverseKL_value_function} is obtained by using Theorem \ref{theorem:supervised_value_function}.
\begin{align}
\label{eq:supervised_reverseKL_value_fucntion}
    \mathcal{J}_{RKL}(D) &= \E_{\mathbf{x} \sim p_{X}(\mathbf{x})} \Biggl[ \E_{\mathbf{y} \sim p_{Y|X}(\mathbf{y}| \mathbf{x})} \Bigl[ - \textbf{D}(\mathbf{y}) \textbf{1}_m(\mathbf{x})  \Bigr] \Biggr] + \E_{ \mathbf{y} \sim p_Y(\mathbf{y})} \Bigl[ \log(\textbf{D}(\mathbf{y})) \textbf{1}_m \Bigr] . \\ \notag
\end{align}

\subsection{HD-Based Objective Functions}
\label{subsubsec:appendix_HD_obj_fcn}
\textbf{Hellinger Squared Distance}: Theorem \ref{theorem:top_down} leads to the variational representation of the Hellinger squared distance, when substituting $f_{u,HD}^{*}$ listed in Table \ref{tab:f_divergences_table_unsupervised} in \eqref{eq:top_down_variational_representation_posterior}
\begin{align}
\label{eq:HD_value_function}
    \mathcal{J}_{HD}(D) &= \E_{(\mathbf{x}, \mathbf{y}) \sim p_{XY}(\mathbf{x}, \mathbf{y})} \Bigl[- \sqrt{D(\mathbf{x}, \mathbf{y})} \Bigr] - \E_{(\mathbf{x}, \mathbf{y}) \sim p_U(\mathbf{x})p_Y(\mathbf{y})} \Biggl[ |\mathcal{T}_x| \frac{1}{\sqrt{D(\mathbf{x}, \mathbf{y})}} \Biggr] ,
\end{align}
where $T^{\diamond}(\mathbf{x}, \mathbf{y}) = \left( \sqrt{\frac{p_{XY}(\mathbf{x}, \mathbf{y})}{|\mathcal{T}_x|p_U(\mathbf{x})p_Y(\mathbf{y})}}  - 1 \right) \sqrt{\frac{|\mathcal{T}_x|p_U(\mathbf{x})p_Y(\mathbf{y})}{p_{XY}(\mathbf{x}, \mathbf{y})}}$, and $D^{\diamond}(\mathbf{x}, \mathbf{y}) = \frac{p_Y(\mathbf{y})}{p_{XY}(\mathbf{x}, \mathbf{y})}$.
The supervised implementation of the objective function in \eqref{eq:HD_value_function} is achieved by using Theorem \ref{theorem:supervised_value_function}
\begin{align}
    \label{eq:supervised_HD_value_function}
    \mathcal{J}_{HD}(D) &= \E_{\mathbf{x} \sim p_{X}(\mathbf{x})} \Biggl[ \E_{\mathbf{y} \sim p_{Y|X}(\mathbf{y}| \mathbf{x})} \Bigl[ -\sqrt{\textbf{D}(\mathbf{y})} \textbf{1}_m(\mathbf{x})  \Bigr] \Biggr] - \E_{ \mathbf{y} \sim p_Y(\mathbf{y})} \Biggl[ \frac{1}{\sqrt{\textbf{D}(\mathbf{y})}} \textbf{1}_m \Biggr] . \\ \notag
\end{align}

\subsection{GAN-Based Objective Functions}
\label{subsubsec:appendix_GAN_obj_fcn}
\textbf{GAN}: Theorem \ref{theorem:top_down} leads to the variational representation of the GAN divergence, when substituting $f_{u,GAN}^{*}$ listed in Table \ref{tab:f_divergences_table_unsupervised} in \eqref{eq:top_down_variational_representation_posterior}
\begin{align}
\label{eq:GAN_value_function}
    \mathcal{J}_{GAN}(D) &= \E_{(\mathbf{x}, \mathbf{y}) \sim p_{XY}(\mathbf{x}, \mathbf{y})} \Bigl[\log(1 - D(\mathbf{x}, \mathbf{y}) )\Bigr] + \E_{(\mathbf{x}, \mathbf{y}) \sim p_U(\mathbf{x})p_Y(\mathbf{y})} \Bigl[ |\mathcal{T}_x| \log(D(\mathbf{x}, \mathbf{y})) \Bigr] ,
\end{align}
where $T^{\diamond}(\mathbf{x}, \mathbf{y}) = \log{\left( \frac{p_{XY}(\mathbf{x}, \mathbf{y})}{p_{XY}(\mathbf{x}, \mathbf{y}) + |\mathcal{T}_x|p_U(\mathbf{x})p_Y(\mathbf{y})} \right)}$ and $D^{\diamond}(\mathbf{x}, \mathbf{y})= \frac{p_Y(\mathbf{y})}{p_Y(\mathbf{y}) + p_{XY}(\mathbf{x}, \mathbf{y})}$. The supervised implementation of the objective function in \eqref{eq:GAN_value_function} is
\begin{align}
    \label{eq:supervised_GAN_value_function}
    \mathcal{J}_{GAN}(D) &= \E_{\mathbf{x} \sim p_{X}(\mathbf{x})} \Biggl[ \E_{\mathbf{y} \sim p_{Y|X}(\mathbf{y}| \mathbf{x})} \Bigl[ \log(\textbf{1}_m - \textbf{D}(\mathbf{y})) \cdot \textbf{1}_m(\mathbf{x})  \Bigr] \Biggr] \notag + \E_{ \mathbf{y} \sim p_Y(\mathbf{y})} \Bigl[ \log(\textbf{D}(\mathbf{y})) \textbf{1}_m \Bigr] . \\ 
\end{align}

\subsection{P-Based Objective Functions}
\label{subsubsec:appendix_P_obj_fcn}
\textbf{Pearson $\chi^2$}: Theorem \ref{theorem:top_down} leads to the variational representation of the Pearson $\chi^2$ divergence, when substituting $f_{u,P}^{*}$ listed in Table \ref{tab:f_divergences_table_unsupervised} in \eqref{eq:top_down_variational_representation_posterior}
\begin{align}
\label{eq:P_value_function}
    \mathcal{J}_{P}(D) &= \E_{(\mathbf{x}, \mathbf{y}) \sim p_{XY}(\mathbf{x}, \mathbf{y})} \Bigl[ 2 (D(\mathbf{x}, \mathbf{y}) - 1) \Bigr] - \E_{(\mathbf{x}, \mathbf{y}) \sim p_U(\mathbf{x})p_Y(\mathbf{y})} \Biggl[ |\mathcal{T}_x| D^2(\mathbf{x}, \mathbf{y}) \Biggr] ,
\end{align}
where $T^{\diamond}(\mathbf{x}, \mathbf{y}) = 2 \left( \frac{p_{XY}(\mathbf{x}, \mathbf{y})}{|\mathcal{T}_x|p_U(\mathbf{x})p_Y(\mathbf{y})} - 1 \right)$, and $D^{\diamond}(\mathbf{x}, \mathbf{y}) = \frac{p_{XY}(\mathbf{x}, \mathbf{y})}{p_Y(\mathbf{y})}$.
The supervised implementation of the objective function in \eqref{eq:P_value_function} is achieved by using Theorem \ref{theorem:supervised_value_function}
\begin{align}
    \label{eq:supervised_HD_value_function}
    \mathcal{J}_{P}(D) &= \E_{\mathbf{x} \sim p_{X}(\mathbf{x})} \Biggl[ \E_{\mathbf{y} \sim p_{Y|X}(\mathbf{y}| \mathbf{x})} \Bigl[ 2(\textbf{D}(\mathbf{y}) - \textbf{1}_m) \textbf{1}_m(\mathbf{x})  \Bigr] \Biggr] - \E_{ \mathbf{y} \sim p_Y(\mathbf{y})} \Biggl[ \textbf{D}(\mathbf{y})\textbf{D}^T(\mathbf{y}) \Biggr] , \\ \notag
\end{align}
where $\textbf{D}(\mathbf{y})\textbf{D}^T(\mathbf{y}) = \sum_{i=1}^m D^2(\mathbf{x}_i, \mathbf{y})$.  

\subsection{SL-Based Objective Functions}
\label{subsubsec:appendix_SL_obj_fcn}
The unsupervised and supervised versions of the objective function corresponding to the shifted log divergence are discussed in this paragraph. For completeness, we report here the generator function and its Fenchel conjugate.\\ $f_{u, SL}(u) = - |\mathcal{T}_x| \log(u + |\mathcal{T}_x|)$, $f_{u, SL}^{*}(t) = -|\mathcal{T}_x|(\log(-t) + t)$.

\textbf{Shifted log}: Theorem \ref{theorem:Bottom-up} leads to the objective function in \eqref{eq:fNOME_cost_fcn} when 
\begin{align}
    k(p_{X|Y}(\mathbf{x}|\mathbf{y})) &= \frac{|\mathcal{T}_x|p_U(\mathbf{x})p_Y(\mathbf{y})}{|\mathcal{T}_x|p_U(\mathbf{x})p_Y(\mathbf{y}) + p_{XY}(\mathbf{x},\mathbf{y})} \\
    g_1(D,k; 1, 0) & = - \frac{|\mathcal{T}_x|p_U(\mathbf{x})p_Y(\mathbf{y}) + p_{XY}(\mathbf{x}, \mathbf{y})}{D(\mathbf{x}, \mathbf{y})} .
\end{align} 
We report here for completeness the unsupervised objective function in \eqref{eq:fNOME_cost_fcn}
\begin{align}
\label{eq:SL_value_function}
    \mathcal{J}_{SL}(D) & = - \E_{(\mathbf{x},\mathbf{y}) \sim p_{XY}(\mathbf{x}, \mathbf{y})} \Bigl[ D(\mathbf{x}, \mathbf{y}) \Bigr] + \E_{(\mathbf{x},\mathbf{y}) \sim p_U(\mathbf{x})p_Y(\mathbf{y})} \Bigl[ |\mathcal{T}_x| \Bigl( \log(D(\mathbf{x}, \mathbf{y})) -  D(\mathbf{x}, \mathbf{y}) \Bigr) \Bigr] ,
\end{align}
which can be obtained from Theorem \ref{theorem:top_down} with $T^{\diamond}(\mathbf{x}, \mathbf{y}) = -\frac{|\mathcal{T}_x| p_U(\mathbf{x})p_Y(\mathbf{y})}{|\mathcal{T}_x|p_U(\mathbf{x})p_Y(\mathbf{y}) + p_{XY}(\mathbf{x}, \mathbf{y})}$ and $D^{\diamond}(\mathbf{x}, \mathbf{y}) = \frac{p_Y(\mathbf{y})}{p_Y(\mathbf{y}) + p_{XY}(\mathbf{x}, \mathbf{y})}$.
The supervised implementation of the objective function in \eqref{eq:SL_value_function} is achieved by using Theorem \ref{theorem:supervised_value_function}
\begin{align}
    \label{eq:supervised_SL_value_function}
    \mathcal{J}_{SL}(D) & = - \E_{\mathbf{x} \sim p_{X}(\mathbf{x})} \Biggl[ \E_{\mathbf{y} \sim p_{Y|X}(\mathbf{y}| \mathbf{x})} \Bigl[ \textbf{D}(\mathbf{y})^T \textbf{1}_m(\mathbf{x})\Bigr] \Biggr] + \E_{\mathbf{y} \sim p_Y(\mathbf{y})} \Bigl[ \Bigl( \log( \textbf{D}(\mathbf{y})) -  \textbf{D}(\mathbf{y}) \Bigr)^T \textbf{1}_m \Bigr] .
\end{align}

\section{Appendix: Proofs}
\label{subsec:Appendix_proofs}
 \stepcounter{customcounter}
 \stepcounter{customcounter}
 \stepcounter{customcounter}
\subsection{Proof of Lemma B.1}
\label{subsec:appendix_f_star_first}

\begin{lemma}
\label{lemma:f_star_first}
    Let $f: \mathbb{R}_+ \longrightarrow \mathbb{R}$ be the generator function of any $f$-divergence. Then, it holds:
    \begin{equation}
    \label{eq:f_equivalence}
        (f^{*})'(t) = (f^{'})^{-1}(t)
    \end{equation}
\end{lemma}
\begin{proof}
Let us recall the definition of Fenchel conjugate, to report a self-contained proof:
    \begin{equation}
f^*(t) = \sup_{u \in \mathbb{R}} \left\{ ut - f(u) \right\} .
\end{equation}
Then, in order to find $\hat{u}$ that achieves the supremum, we impose 
\begin{equation}
\label{eq:fenchel_derivative_0}
    \frac{\partial}{\partial u} \left\{ ut- f(u) \right\} = 0 ,
\end{equation}
that implies $f^{'}(u) = t$. The condition \eqref{eq:fenchel_derivative_0} can be imposed because $f(\cdot)$ is a convex function. 
Thus, 
\begin{equation}
\label{eq:u_hat}
    \hat{u} = (f^{'})^{-1}(t)
\end{equation}
Then, substituting \eqref{eq:u_hat} in the definition of the fenchel conjugate, it becomes:
\begin{equation}
    f^*(t) = (f^{'})^{-1}(t)t - f((f^{'})^{-1}(t)) .
\end{equation}
Then, by computing the first derivative w.r.t. $t$:
\begin{equation}
    (f^{*})^{'}(t) = ((f^{'})^{-1})^{'}(t)t + (f^{'})^{-1}(t) - \underbrace{f^{'}((f^{'})^{-1}(t))}_{=t} ((f^{'})^{-1})^{'}(t) .
\end{equation}
The first and third terms cancel out, leading to \eqref{eq:f_equivalence}.
\end{proof}

\subsection{Proof of Lemma B.2}
\begin{lemma}
\label{lemma:f_u_valid_divergence}
    Let $f: \mathbb{R}_+ \longrightarrow \mathbb{R}$ be the generator function of any $f$-divergence, $f^*(\cdot)$ its Fenchel conjugate, and $K>0$ a constant. Then, $f^*_u(t) \triangleq K f^*(t)$ is the Fenchel conjugate of a valid $f$-divergence.
\end{lemma}
\begin{proof}
    Firstly, we must prove the convexity of $f_u(u)$. Since $f^*(t)$ is a convex function, then also $f^*_u(t)$ is a convex function because $K$ is a positive constant (if the second derivative of $f^*(t)$ is non-negative, then multiplying it by a positive constant will result in a non-negative function). If $f^*_u(t)$ is a convex function, then also $f_u(u)$ is convex, by definition, because it is computed as the Fenchel conjugate (i.e., the convex conjugate). \\
    Secondly, if $f_u(1)=C$, then $f_u(1)=0$ is achieved by subtracting $C$ to the $f_u(u)$ obtained from the computation of the Fenchel conjugate of $f_u^*(t)$. The subtraction of $C$ to $f_u(u)$ has just a translation effect, not affecting the training process. To prove it, let consider $f_u(u) = f_u^{\bullet}(u)-C$, with $f_u(1)=0$ (thus $f_u^\bullet(1)=C$). 
    Then, 
    \begin{align}
        f^*_u(t)= \sup_u[ut - f_u(u)] = \sup_u[ut-f_u^\bullet(u) + C]
    \end{align}
    The $u$ that maximizes $ut-f_u^\bullet(u) + C$ (referred to as $\hat{u}$) is obtained by imposing the first derivative w.r.t. $u$ equal to zero (similarly to \eqref{eq:fenchel_derivative_0} and \eqref{eq:u_hat} in the proof of Lemma \ref{lemma:f_star_first}). Therefore, $\hat{u}$ is not influenced by $C$, since $C$ disappears when computing the first derivative. Then, the value of $\hat{u}$ is substituted in $ut-f_u^\bullet(u) + C$. Therefore, we obtain $f^*_u(t) = \hat{u}t-f_u^\bullet(\hat{u}) + C$, where the first two terms $\hat{u}t$ and $f_u^\bullet(\hat{u})$ do not depend on $C$. Thus, $C$ becomes a constant in $f_u^*(t)$, which implies that it becomes an additive constant in the objective function (see \eqref{eq:top_down_variational_representation_posterior}), which does not have any effect on the training procedure.
    Finally, $f_u(1)=C$ (thus $f^*_u(t) = \hat{u}t-f_u^\bullet(\hat{u}) + C$) does not impact the estimate of $p_{X|Y}$, since it is computed using $(f_u^*)^{\prime}(\cdot)$, which is not affected by $C$.
\end{proof}

\subsection{Proof of Theorem 3.1}
\label{subsec:appendix_theorem_top_down}
\begin{theoremappendix}
Let $X$ and $Y$ be the random vectors with probability density functions $p_X(\mathbf{x})$ and $p_Y(\mathbf{y})$, respectively. Assume $\mathbf{y} = H(\mathbf{x})$, where $H(\cdot)$ is a stochastic function, then $p_{XY}(\mathbf{x}, \mathbf{y})$ is the joint density. Define $\mathcal{T}_x$ to be the support of $X$ and $p_U(\mathbf{x})$ to be a uniform distribution with support $\mathcal{T}_x$. Let $f_{u}: \mathbb{R}_+ \longrightarrow \mathbb{R}$ be a convex function such that $f_{u}(1)=0$, and $f_{u}^*$ be the Fenchel conjugate of $f_{u}$.
Let $\mathcal{J}_f(T)$ be the objective function defined as 
\begin{align}
\label{eq:top_down_variational_representation_posterior_appendix}
    \mathcal{J}_f(T) &= \E_{(\mathbf{x},\mathbf{y}) \sim p_{XY}(\mathbf{x},\mathbf{y})} \left[ T(\mathbf{x},\mathbf{y}) \right] - \E_{(\mathbf{x},\mathbf{y}) \sim p_U(\mathbf{x})p_Y(\mathbf{y})} \left[ f_{u}^*(T(\mathbf{x},\mathbf{y})) \right] .
\end{align}
Then, 
\begin{equation}
    T^{\diamond}(\mathbf{x},\mathbf{y}) = \argmax_{T \in \mathcal{T}} \mathcal{J}_f(T) 
\end{equation}
leads to the estimation of the posterior density
\begin{equation}
\label{eq:top_down_posterior_estimator_appendix}
    \hat{p}_{X|Y}(\mathbf{x}|\mathbf{y}) = \frac{p_{XY}(\mathbf{x},\mathbf{y})}{p_Y(\mathbf{y})} = \frac{(f_{u}^{*})^{\prime}(T^{\diamond}(\mathbf{x},\mathbf{y}))}{|\mathcal{T}_x|},
\end{equation}
where $T^{\diamond}(\mathbf{x},\mathbf{y})$ is parametrized by an artificial neural network.
\end{theoremappendix}
\begin{proof} 
From \cite{Nguyen2010}, $T^{\diamond}(\mathbf{x}, \mathbf{y})$ achieved when maximizing \eqref{eq:top_down_variational_representation_posterior_appendix} is
\begin{equation}
\label{eq:t_opt_top_down_appendix}
    T^{\diamond}(\mathbf{x}, \mathbf{y}) = f^{\prime}_u \left( \frac{p_{XY}(\mathbf{x}, \mathbf{y})}{p_U(\mathbf{x})p_Y(\mathbf{y})}\right) ,
\end{equation}
as defined in \eqref{eq:T_hat}.
Thus, \eqref{eq:top_down_posterior_estimator_appendix} is equivalent to
\begin{align}
\label{eq:top_down_proof_first_step_appendix}
    \frac{(f_{u}^{*})^{\prime}(T^{\diamond}(\mathbf{x},\mathbf{y}))}{|\mathcal{T}_x|} &= \frac{(f_u^\prime)^{-1}(T^{\diamond}(\mathbf{x},\mathbf{y}))}{|\mathcal{T}_x|}\\
    \label{eq:top_down_proof_second_step_appendix}
    &= \frac{(f_u^\prime)^{-1} \Bigl(f^{\prime}_u \left( \frac{p_{XY}(\mathbf{x}, \mathbf{y})}{p_U(\mathbf{x})p_Y(\mathbf{y})}\right) \Bigr)}{|\mathcal{T}_x|} \\
    &= \frac{p_{XY}(\mathbf{x}, \mathbf{y})}{|\mathcal{T}_x|p_U(\mathbf{x})p_Y(\mathbf{y})}\\
    &= \frac{p_{XY}(\mathbf{x}, \mathbf{y})}{p_Y(\mathbf{y})} ,
\end{align}
where the equality in \eqref{eq:top_down_proof_first_step_appendix} is proved in Lemma \ref{lemma:f_star_first} in Appendix \ref{subsec:Appendix_proofs}, while \eqref{eq:top_down_proof_second_step_appendix} is obtained by substituting \eqref{eq:t_opt_top_down_appendix} in \eqref{eq:top_down_proof_first_step_appendix}. Since $p_U(\mathbf{x}) = \frac{1}{|\mathcal{T}_x|}$, as it is the uniform probability density function (pdf) over $\mathcal{T}_x$, the thesis follows. From \cite{tonello2022mind}, the usage of the uniform probability density function $p_U(\mathbf{x})$ is fundamental to define the objective function in \eqref{eq:top_down_variational_representation_posterior_appendix}. Its importance derives from the need of the discriminator to be fed with both $X$ and $Y$ realizations.
\end{proof}

\subsection{Proof of Lemma 3.2}
\label{subsec:appendix_lemma_convergence}
\begin{lemmaappendix}
    Let the artificial neural network $D(\cdot) \in \mathcal{D}$ be with enough capacity and training time (i.e., in the nonparametric limit). Assume the gradient ascent update rule $D^{(i+1)} = D^{(i)} + \mu \nabla \mathcal{J}_f(D^{(i)})$ converges to
    \begin{equation}
        D^{\diamond} = \argmax_{D \in \mathcal{D}} \mathcal{J}_f(D) ,
    \end{equation}
    where $\mathcal{J}_f(D)$ is defined as in \eqref{eq:top_down_variational_representation_posterior}, with the change of variable $D = r^{-1}(T)$.
    Then, the difference between the optimal posterior probability and its estimate at iteration $i$ is
    \begin{align}
        p^{\diamond} - p^{(i)} \simeq \frac{1}{|\mathcal{T}_x|} \Bigl( \delta^{(i)} \Bigl[ (f_{u}^{*})^{\prime \prime}(r(D^{(i)})) \Bigr] \Bigr) , \label{eq:gap_optimum_posterior_appendix}
    \end{align}
    where $\delta^{(i)}=r(D^{\diamond}) - r(D^{(i)})$, and $\mu >0$ the learning rate.
    If $D^{\diamond}$ corresponds to the global optimum achieved by using the gradient ascent method, the posterior probability estimator in \eqref{eq:top_down_posterior_estimator} converges to the real value of the posterior density.
\end{lemmaappendix}
\begin{proof}
    The proof follows a procedure similar to Lemma 3 in \cite{letizia2023variational}.
    We define $\delta^{(i)} = r(D^{\diamond}) - r(D^{(i)})$ as the difference between the optimum $T^{\diamond}$ and the one achieved at the $i^{th}$ iteration of the training procedure, when using a gradient ascent update method. Define $p^{(i)}$ and $p^{\diamond}$ as the posterior probability estimate at iteration $i$ and the optimum, respectively. Then,
    \begin{align}
        p^{\diamond} - p^{(i)} &= \frac{1}{|\mathcal{T}_x|} \Bigl((f_{u}^{*})^{\prime} (r(D^{\diamond})) - (f_{u}^{*})^{\prime}(r(D^{(i)})) \Bigr) \\
        &= \frac{1}{|\mathcal{T}_x|} \Bigl( (f_{u}^{*})^{\prime} (r(D^{\diamond})) - (f_{u}^{*})^{\prime}(r(D^{\diamond}) - \delta^{(i)}) \Bigr) \\
        & \simeq \frac{1}{|\mathcal{T}_x|} \Bigl( \delta^{(i)} \Bigl[ (f_{u}^{*})^{\prime \prime}(r(D^{\diamond}) - \delta^{(i)}) \Bigr] \Bigr) , \label{eq:gap_optimum_posterior_appendix_penultimo}
    \end{align}
    where the last step develops from the first order Taylor expansion in $r(D^{\diamond}) - \delta^{(i)}$. From \eqref{eq:gap_optimum_posterior_appendix_penultimo}, the thesis in \eqref{eq:gap_optimum_posterior_appendix} follows.
    If the gradient ascent method converges towards the maximum, $\delta^{(i)} \longrightarrow 0$. Thus, when $i \longrightarrow \infty$, $|p^{\diamond} - p^{(i)}| \longrightarrow 0$.
\end{proof}

\subsection{Proof of Theorem 4.1}
\label{subsec:appendix_theorem_bottom_up}
\stepcounter{customcounter}
\begin{theoremappendix}  
Let $X$ and $Y$ be the random vectors with probability density functions $p_X(\mathbf{x})$ and $p_Y(\mathbf{y})$, respectively. Assume $\mathbf{y} = H(\mathbf{x})$, where $H(\cdot)$ is a stochastic function, then $p_{XY}(\mathbf{x}, \mathbf{y})$ is the joint density. Let $\mathcal{T}_x$ and $\mathcal{T}_y$ be the support of $p_X(\mathbf{x})$ and $p_Y(\mathbf{y})$, respectively.
Let the discriminator $D(\mathbf{x}, \mathbf{y})$ be a scalar function of $\mathbf{x}$ and $\mathbf{y}$. Let $k(\cdot)$ be any deterministic and invertible function.
Then, the posterior density is estimated as
\begin{equation}
\label{eq:bottom_up_posterior_estimator_appendix}
    \hat{p}_{X|Y}(\mathbf{x}|\mathbf{y}) = k^{-1}(D^{\diamond}(\mathbf{x}, \mathbf{y})) ,
\end{equation}
where $D^{\diamond}(\mathbf{x},\mathbf{y})$ is the optimal discriminator obtained by maximizing
\begin{equation}
\label{eq:value_function_general_bottom_up_appendix}
    \mathcal{J}(D) = \int_{\mathcal{T}_x} \int_{\mathcal{T}_y}  \tilde{\mathcal{J}}(D) d\mathbf{x} d\mathbf{y},
\end{equation}
for all concave functions $ \tilde{\mathcal{J}}(D)$ such that their first derivative is 
\begin{align}
\label{eq:g(D)_appendix}
    \frac{\partial \tilde{\mathcal{J}}(D)}{\partial D} &= \Bigl( D(\mathbf{x}, \mathbf{y}) - k(p_{X|Y}(\mathbf{x}|\mathbf{y}))\Bigr) g_1(D,k) \notag \\
    & \triangleq g(D, k)
\end{align}
with $g_1(D,k) \neq 0$ deterministic, and $\frac{\partial g(D,k)}{ \partial D} \leq 0$.
\end{theoremappendix}
\begin{proof}
    A necessary condition to maximize $\mathcal{J}(D)$ requires to set the first derivative of the integrand $\tilde{\mathcal{J}}(D)$ w.r.t. $D$ equal to zero. Since $g_1(D,k) \neq 0$, from \eqref{eq:g(D)_appendix} easily follows 
    \begin{equation}
    \label{eq:D_star_k_appendix}
        D^{\diamond}(\mathbf{x}, \mathbf{y}) = k(p_{X|Y}(\mathbf{x}|\mathbf{y})) = k \left( \frac{p_{XY}(\mathbf{x},\mathbf{y})}{p_Y(\mathbf{y})} \right) .
    \end{equation}
    The concavity of $\tilde{\mathcal{J}}(D)$ is obtained by imposing the first derivative of $g(D,k)$ with respect to $D$ to be nonpositive, i.e.,
    \begin{align}
        0 & \geq \frac{\partial }{\partial D} \Bigl\{ \Bigl( D(\mathbf{x}, \mathbf{y}) - k(p_{X|Y}(\mathbf{x}|\mathbf{y}))\Bigr) g_1(D,k) \Bigr\} \notag \\
        & = g_1(D,k) + \Bigl( D(\mathbf{x}, \mathbf{y}) - k(p_{X|Y}(\mathbf{x}|\mathbf{y}))\Bigr) \frac{\partial g_1(D,k)}{\partial D} .
    \end{align}
    Therefore, the stationary point $D^{\diamond}(\mathbf{x}, \mathbf{y})$ corresponds to a maximum. 
\end{proof}

\subsection{Proof of Corollary 4.2}
\label{subsec:appendix_corollary_bottom_up_top_down_value_functions}
\begin{corollaryappendix}
\label{corollary:bottom-up-top-down-value-functions}
    Let $\mathcal{J}(D)$ be defined as in Theorem \ref{theorem:Bottom-up}.
    Let 
    \begin{align}
        k(p_{X|Y}(\mathbf{x}|\mathbf{y})) = k \left( \frac{p_{XY}(\mathbf{x},\mathbf{y})}{|\mathcal{T}_x|p_U(\mathbf{x})p_Y(\mathbf{y})} \right) = \frac{p_0}{p_1} .
    \end{align}
    Let 
    \begin{equation}
    \label{eq:g_1(D)}
        g_1(D,k; \alpha, \beta) \triangleq -\frac{p_1}{D^{\alpha}} \left( \frac{1}{1-D} \right)^{\beta} ,
    \end{equation}
    where $\alpha, \beta \in \mathbb{Q}$.
    Then, the objective functions in Table \ref{tab:value functions} are obtained from Theorem \ref{theorem:Bottom-up} by using $k(\cdot)$ and $g_1(\cdot)$ as defined in the set $\mathcal{F}_{k, g_1}\triangleq \{(k(x), (\alpha, \beta))_f \}$: 
    \begin{align}
        \mathcal{F}_{k, g_1} = &\Biggl\{ \Biggl(x, \left(1, 0 \right) \Biggr)_{KL}, \Biggl(-\frac{1}{x}, \left(1,0\right)\Biggr)_{RKL}, \Biggl(\frac{1}{x}, \left(\frac{3}{2},0\right)\Biggr)_{HD}, \Biggl(\frac{1}{1+x}, \left(1,1\right)\Biggr)_{GAN}, \Biggl( x, (0,0) \Biggr)_{P} \Biggr\}  
    \end{align}
\end{corollaryappendix}
\begin{proof}
    Let rewrite for completeness 
    \begin{align}
        k_{KL}(p_{X|Y}(\mathbf{x}|\mathbf{y})) &= p_{X|Y}(\mathbf{x}|\mathbf{y}),\\
        k_{RKL}(p_{X|Y}(\mathbf{x}|\mathbf{y})) &= -\frac{1}{p_{X|Y}(\mathbf{x}|\mathbf{y})},\\
        k_{HD}(p_{X|Y}(\mathbf{x}|\mathbf{y})) &= \frac{1}{p_{X|Y}(\mathbf{x}|\mathbf{y})},\\
        k_{GAN}(p_{X|Y}(\mathbf{x}|\mathbf{y})) &= \frac{1}{1 + p_{X|Y}(\mathbf{x}|\mathbf{y})},\\
        k_{P}(p_{X|Y}(\mathbf{x}|\mathbf{y})) &= p_{X|Y}(\mathbf{x}|\mathbf{y}),
    \end{align}
    and 
    \begin{align}
        g_{1_{KL}}(D,k) &= -\frac{p_U(\mathbf{x})p_Y(\mathbf{y})|\mathcal{T}_x|}{D(\mathbf{x}, \mathbf{y})},\\
        g_{1_{RKL}}(D,k) &= \frac{p_{XY}(\mathbf{x}, \mathbf{y})}{D(\mathbf{x}, \mathbf{y})},\\
        g_{1_{HD}}(D,k) &= -\frac{p_{XY}(\mathbf{x}, \mathbf{y})}{D^{\frac{3}{2}}(\mathbf{x}, \mathbf{y})},\\
        g_{1_{GAN}}(D,k) &= -\frac{p_U(\mathbf{x})p_Y(\mathbf{y})|\mathcal{T}_x| + p_{XY}(\mathbf{x}, \mathbf{y})}{D(\mathbf{x}, \mathbf{y})(1-D(\mathbf{x}, \mathbf{y}))},\\
        g_{1_{P}}(D,k) &= -p_U(\mathbf{x})p_Y(\mathbf{y})|\mathcal{T}_x|,
    \end{align}
    where the dependence of $g_1(\cdot)$ from $k$ is intrinsic in $p_1$. For instance, the KL and RKL divergences are obtained by using the same values of $\alpha$ and $\beta$ in \eqref{eq:g_1(D)}, but since $k_{KL}(\cdot) \neq k_{RKL}(\cdot)$, $g_{1_{KL}}(\cdot) \neq g_{1_{RKL}}(\cdot)$ . $g_1(\cdot) \neq 0$ because by definition each pdf is different from $0$ in its support.  
    After a substitution in \eqref{eq:g(D)}, we obtain
    \begin{align}
        g_{KL}(D,k) &= \frac{p_{XY}(\mathbf{x}, \mathbf{y})}{D(\mathbf{x}, \mathbf{y})} - p_U(\mathbf{x})p_Y(\mathbf{y})|\mathcal{T}_x|, \\
        g_{RKL}(D,k) &= - p_{XY}(\mathbf{x}, \mathbf{y}) - \frac{p_Y(\mathbf{y})}{D(\mathbf{x}, \mathbf{y})},\\
        g_{HD}(D,k) &= -\frac{p_{XY}(\mathbf{x}, \mathbf{y})}{\sqrt{D(\mathbf{x}, \mathbf{y})}} + \frac{p_U(\mathbf{x})p_Y(\mathbf{y})}{\sqrt{D^3(\mathbf{x}, \mathbf{y})}},\\
        g_{GAN}(D,k) &= \frac{p_U(\mathbf{x})p_Y(\mathbf{y})|\mathcal{T}_x|}{D(\mathbf{x}, \mathbf{y})} - \frac{p_{XY}(\mathbf{x}, \mathbf{y})}{1-D(\mathbf{x}, \mathbf{y})}, \\
        g_{P}(D,k) &= p_{XY}(\mathbf{x}, \mathbf{y}) - p_U(\mathbf{x})p_Y(\mathbf{y})|\mathcal{T}_x|D(\mathbf{x}, \mathbf{y}).
    \end{align}
    After the integration w.r.t. $D$ and the substitution in \eqref{eq:value_function_general_bottom_up}, the objective functions in Table \ref{tab:value functions} are attained.
\end{proof}

\subsection{Proof of Theorem 5.1}
\label{subsec:appendix_theorem_SL_unsupervised}
\stepcounter{customcounter}
\begin{theoremappendix}
Let $X$ and $Y$ be two random vectors with pdfs $p_X(\mathbf{x})$ and $p_Y(\mathbf{y})$, respectively. Assume $Y = H(X)$, with $H(\cdot)$ stochastic function, then let $p_{XY}(\mathbf{x}, \mathbf{y})$ be the joint density. Let $\mathcal{T}_x$ be the support of $X$. Let $p_U(\mathbf{x})$ be a uniform pdf having the same support $\mathcal{T}_x$.  
The maximization of the objective function 
\begin{align}
\label{eq:fNOME_cost_fcn_appendix}
    \mathcal{J}_{SL}(D) & = - \E_{(\mathbf{x},\mathbf{y}) \sim p_{XY}(\mathbf{x}, \mathbf{y})} \Bigl[ D(\mathbf{x}, \mathbf{y}) \Bigr] + \E_{(\mathbf{x},\mathbf{y}) \sim p_U(\mathbf{x})p_Y(\mathbf{y})} \Bigl[ |\mathcal{T}_x| \Bigl( \log(D(\mathbf{x}, \mathbf{y})) -  D(\mathbf{x}, \mathbf{y}) \Bigr) \Bigr],
\end{align}
leads to the optimal discriminator output 
\begin{equation}
\label{eq:fNAME_D_opt_appendix}
    D^{\diamond}(\mathbf{x}, \mathbf{y}) = \argmax_{D} \mathcal{J}_{SL}(D) = \frac{1}{1 + p_{X|Y}(\mathbf{x}, \mathbf{y})} ,
\end{equation}
and the posterior density estimate is computed as
\begin{equation}
\label{eq:posterior_estimator_sl_appendix}
    \hat{p}_{X|Y}(\mathbf{x}|\mathbf{y}) = \frac{1 - D^{\diamond}(\mathbf{x},\mathbf{y})}{D^{\diamond}(\mathbf{x},\mathbf{y})} .
\end{equation}
\end{theoremappendix}
\begin{proof}
    Following Theorem \ref{theorem:Bottom-up}, the proof starts by inverting \eqref{eq:bottom_up_posterior_estimator}. Specifically, we set $k(x) = \frac{1}{1+x}$ and from \eqref{eq:fNAME_D_opt_appendix}, by expressing the posterior density as the density ratio $p_{XY}/p_Y$, we achieve
    \begin{equation}
        D^{\diamond}(\mathbf{x}, \mathbf{y}) - \frac{|\mathcal{T}_x|p_U(\mathbf{x})p_Y(\mathbf{y})}{|\mathcal{T}_x|p_U(\mathbf{x})p_Y(\mathbf{y}) + p_{XY}(\mathbf{x},\mathbf{y})} = 0 ,
    \end{equation}
    where $p_U(\mathbf{x}) = 1/|\mathcal{T}_x|$.
    Then, \eqref{eq:g_1(D)} with $\alpha=1$ and $\beta=0$ becomes
    \begin{align}
    \label{eq:g_1(D)_SL}
        g_1(D,k; 1, 0) & = - \frac{|\mathcal{T}_x|p_U(\mathbf{x})p_Y(\mathbf{y}) + p_{XY}(\mathbf{x}, \mathbf{y})}{D(\mathbf{x}, \mathbf{y})} .
    \end{align}
    Then, \eqref{eq:g_1(D)_SL} is substituted in \eqref{eq:g(D)}, obtaining
    \begin{align}
    \label{eq:partial_der_fNome}
         g(D,k) = \frac{\partial }{\partial D} \mathcal{\tilde{J}}(D) =& - \Bigl( |\mathcal{T}_x|p_U(\mathbf{x})p_Y(\mathbf{y}) + p_{XY}(\mathbf{x},\mathbf{y}) \Bigr) + \frac{|\mathcal{T}_x|p_U(\mathbf{x})p_Y(\mathbf{y})}{D(\mathbf{x}, \mathbf{y})} .
    \end{align}
    The computation of the integral of \eqref{eq:partial_der_fNome} with respect to the discriminator's output $D$ leads to
    \begin{align}
    \tilde{\mathcal{J}}(D) &= - p_{XY}(\mathbf{x}, \mathbf{y}) D(\mathbf{x}, \mathbf{y}) +  |\mathcal{T}_x| p_U(\mathbf{x})p_Y(\mathbf{y}) \Bigl( \log(D(\mathbf{x}, \mathbf{y})) - D(\mathbf{x}, \mathbf{y}) \Bigr)  ,
    \end{align}
    which proves the statement of the theorem, since 
    \begin{equation}
        \frac{\partial g(D,k)}{\partial D} = -\frac{p_Y(\mathbf{y})}{D^2(\mathbf{x}, \mathbf{y})} \leq 0 .
    \end{equation}
    Given the optimum discriminator $D^{\diamond}$, the posterior density estimator in \eqref{eq:posterior_estimator_sl_appendix} is achieved by inverting \eqref{eq:fNAME_D_opt_appendix}. 
\end{proof}

\subsection{Proof of Corollary 5.2}
\label{subsec:appendix_corollary_f_SL}
\begin{corollaryappendix}
    Define the generator function 
\begin{equation}
    f_{u, SL}(u) = - |\mathcal{T}_x| \log(u + |\mathcal{T}_x|) + K ,
\end{equation}
where $K=|\mathcal{T}_x| \log(1 + |\mathcal{T}_x|)$ is constant. Then, $\mathcal{J}_{SL}(D)$ in \eqref{eq:fNOME_cost_fcn} is the variational representation of $D_{f_{u,SL}}(p_{XY}||p_Up_Y)$.
\end{corollaryappendix}
\begin{proof}
    By comparing \eqref{eq:fNOME_cost_fcn} to \eqref{eq:top_down_variational_representation_posterior},  it is immediate to notice the change of variable $D(\mathbf{x}, \mathbf{y}) = - T(\mathbf{x}, \mathbf{y})$. Then, the expression of the Fenchel conjugate is attained as $f_{u, SL}^{*}(t) = -|\mathcal{T}_x|(\log(-t) + t)$ from inspection of the expectation over $p_U(\mathbf{x})p_Y(\mathbf{y})$. The generating function $f_{u, SL}(u) = - |\mathcal{T}_x| \log(u + |\mathcal{T}_x|)$ is computed by using the definition of Fenchel conjugate. We add a constant $K = |\mathcal{T}_x|\log(1+|\mathcal{T}_x|)$ to the generator function to achieve the condition $f_{u, SL}(1)=0$, which has no impact on the maximization of the objective function in \eqref{eq:fNOME_cost_fcn}.
    Lastly, the second derivatives of $f_{u, SL}(u)$ and $f_{u, SL}^{*}(t)$ are nonpositive functions, proving that the generator function and its Fenchel conjugate are convex.
\end{proof}

\subsection{Proof of Corollary 5.3}
\label{subsec:appendix_corollary_upper_lower_bound_SL}
\begin{corollaryappendix}
\label{corollary:upper_lower_bound_newF}
    Let $P$ and $Q$ be two probability distributions. 
    Let $D_{SL}(P||Q)$ be the $f$-divergence with generator function $f_{SL}(u)$ as in \eqref{eq:general_fSL}. Then, 
    \begin{equation}
        0 \leq D_{SL}(P||Q) \leq \log(2) .
    \end{equation}
\end{corollaryappendix}
\begin{proof}
    Let $f: \mathbb{R}_+ \longrightarrow \mathbb{R}$ be a convex function with $f(1)=0$, and $f^{\circ}: \mathbb{R}_+ \longrightarrow \mathbb{R}$ defined as
    \begin{equation}
        f^{\circ}(u) \triangleq u f \left( \frac{1}{u} \right),
    \end{equation}
    then $f^{\circ}$ is also convex and such that $f^{\circ}(1)=0$.
    Then, the \textit{Range of Values Theorem} \cite{vajda1972f} sets upper and lower bounds on the value of the $f$-divergence between two distributions $P$ and $Q$, depending on $f(u)$ and $f^{\circ}(u)$:
\begin{equation}
\label{eq:range_of_values}
    f(1) \leq D_f(P||Q) \leq f(0) + f^{\circ}(0) \>\>\>\>\> \forall Q, P .
\end{equation}
From which the thesis follows.
\end{proof}

\subsection{Proof of Corollary 5.4}
\label{subsec:appendix_corollary_speed_convergence}
\begin{corollaryappendix}
        Let $\mathcal{J}_{SL}(D)$ be defined as in \eqref{eq:fNOME_cost_fcn}. Let $\mathcal{J}_{GAN}(D)$ be defined as in Table \ref{tab:value functions}. Let $D^{\diamond}_N$ be the discriminator output in a neighborhood of $D^{\diamond}$ where $\mathcal{J}_{SL}(D^{\diamond}_N)$ and $\mathcal{J}_{GAN}(D^{\diamond}_N)$ are concave.
        Then,
        \begin{equation}
        \label{eq:unsupervised_GAN_SL}
            \Bigg| \frac{\partial \mathcal{J}_{GAN}(D^{\diamond}_N)}{\partial D } \Bigg| \geq \Bigg| \frac{\partial \mathcal{J}_{SL}(D^{\diamond}_N)}{\partial D } \Bigg| .
        \end{equation}
    \end{corollaryappendix}
    \begin{proof}
        To prove \eqref{eq:unsupervised_GAN_SL}, we just need to prove 
        \begin{equation}
            \Bigg|\frac{\partial \tilde{\mathcal{J}}_{GAN}(D^{\diamond}_N)}{\partial D} \Bigg| \geq \Bigg| \frac{\partial \tilde{\mathcal{J}}_{SL}(D^{\diamond}_N)}{\partial D} \Bigg| ,
        \end{equation}
        (where $\tilde{\mathcal{J}}$ is the integrand function in \eqref{eq:value_function_general_bottom_up}) since the inequality between the integrands holds when the integrals are computed over the same interval.\\
        Lemma \ref{lemma:convergence} guarantees the convergence to the optimal discriminator. Therefore, let $D^{\diamond}_N = \frac{p_Y(\mathbf{y})}{p_{XY}(\mathbf{x}, \mathbf{y}) + p_Y(\mathbf{y})} + \delta$, with $\delta$ arbitrarily small, so that $D^{\diamond}_N$ belongs to the neighborhood of $D^{\diamond}$. Then,
        \begin{align}
        \label{eq:first_derivative_SL}
            \frac{\partial}{\partial D} \tilde{\mathcal{J}}_{SL}(D) \Bigg|_{D^{\diamond}_N} &= - (p_{XY}(\mathbf{x}, \mathbf{y}) + p_Y(\mathbf{y})) + \frac{p_Y(\mathbf{y})}{D^{\diamond}_N(\mathbf{x}, \mathbf{y})} \\
            \frac{\partial}{\partial D}\tilde{\mathcal{J}}_{GAN}(D) \Bigg|_{D^{\diamond}_N} &= - \frac{p_{XY}(\mathbf{x}, \mathbf{y})}{1-D^{\diamond}_N(\mathbf{x}, \mathbf{y})} + \frac{p_Y(\mathbf{y})}{D^{\diamond}_N(\mathbf{x}, \mathbf{y})} \notag \\
            &= - \frac{p_{XY}(\mathbf{x}, \mathbf{y})(p_{XY}(\mathbf{x}, \mathbf{y}) + p_Y(\mathbf{y}))}{p_{XY}(\mathbf{x}, \mathbf{y}) (1- \delta) - \delta p_Y(\mathbf{y})} \notag \\
            & + \frac{p_Y(\mathbf{y})}{D^{\diamond}_N(\mathbf{x}, \mathbf{y})
            \label{eq:first_derivative_GAN_long}}
        \end{align}
        By substituting $\gamma = \frac{1}{1-\delta}$, \eqref{eq:first_derivative_GAN_long} becomes
        \begin{align}
        \label{eq:first_derivative_GAN}
            \frac{\partial}{\partial D}\tilde{\mathcal{J}}_{GAN}(D) \Bigg|_{D^{\diamond}_N} & \approx - \gamma (p_{XY}(\mathbf{x}, \mathbf{y}) + p_Y(\mathbf{y})) + \frac{p_Y(\mathbf{y})}{D^{\diamond}_N(\mathbf{x}, \mathbf{y})} ,
        \end{align}
        where $0 < \gamma < 1$ if $\delta < 0$, and $\gamma > 1$ if $\delta >0$. 
        Thus, the comparison between \eqref{eq:first_derivative_SL} and \eqref{eq:first_derivative_GAN} leads to the inequalities
        \begin{align}
            \frac{\partial}{\partial D} \tilde{\mathcal{J}}_{SL}(D) \Bigg|_{D^{\diamond}_N} &< \frac{\partial}{\partial D}\tilde{\mathcal{J}}_{GAN}(D) \Bigg|_{D^{\diamond}_N} \quad if \delta < 0 \\
            \frac{\partial}{\partial D} \tilde{\mathcal{J}}_{SL}(D) \Bigg|_{D^{\diamond}_N} &> \frac{\partial}{\partial D}\tilde{\mathcal{J}}_{GAN}(D) \Bigg|_{D^{\diamond}_N} \quad if \delta > 0 .
        \end{align}
        Since $D^{\diamond}$ corresponds to a maximum, i.e., the sign of the left derivative is positive, and the sign of the right derivative is negative, the statement of the corollary is proved.
    \end{proof}

\subsection{Proof of Theorem 6.1}
\label{subsec:appendix_theorem_supervised_obj_fcn}
\stepcounter{customcounter}
\begin{theoremappendix}
    Let $p_X(\mathbf{x})$ and $p_Y(\mathbf{y})$ be pdfs describing the input and output of a stochastic function $H(\cdot)$, respectively. Let $p_X(\mathbf{x}) \triangleq \sum_{i = 1}^m P_X(\mathbf{x}_i) \delta (\mathbf{x} - \mathbf{x}_i)$, where $P_X(\cdot)$ is the probability mass function of $X$. Let $\mathcal{T}_x$ be the support of $p_X(\mathbf{x})$ and $|\mathcal{T}_x|$ its Lebesgue measure.
    Let $p_U(\mathbf{x})$ be the uniform discrete probability density function over $\mathcal{T}_x$. Let the discriminator be characterized by a supervised architecture. Then, the objective function in \eqref{eq:top_down_variational_representation_posterior} becomes
    \begin{align}
    \label{eq:supervised_general_value_function_appendix}
        \mathcal{J}(D) &= \E_{\mathbf{x} \sim p_X(\mathbf{x})}\Biggl[ \E_{\mathbf{y} \sim p_{Y|X}(\mathbf{y}|\mathbf{x})}\Bigl[ r(\textbf{D}(\mathbf{y}))^T \textbf{1}_m(\mathbf{x}) \Bigr] \Biggr] - \E_{\mathbf{y} \sim p_Y(\mathbf{y})}\Biggl[ \sum_{i=1}^m f^{*} \left(r(D(\mathbf{x}_i,\mathbf{y})) \right) \Biggr] ,
    \end{align}
    where $D(\mathbf{x}_i,\mathbf{y})$ is the i-th component of $\textbf{D}(\mathbf{y})$ and $\textbf{T}(\mathbf{y}) = r(\textbf{D}(\mathbf{y}))$.
\end{theoremappendix}
\begin{proof}
    Let the alphabet of $X$ be $\mathcal{A}_x = \left\{ \mathbf{x}_1, \dots, \mathbf{x}_m \right\}$.
    The objective function in \eqref{eq:top_down_variational_representation_posterior} can be expressed as
    \begin{align}
        \mathcal{J}(D) &= \int_{\mathcal{T}_y} \int_{\mathcal{T}_x} p_X(\mathbf{x}) p_{Y|X}(\mathbf{y}|\mathbf{x}) r(D(\mathbf{x},\mathbf{y})) - p_U(\mathbf{x}) p_Y(\mathbf{y}) f_{u}^{*} \left( r(D(\mathbf{x}, \mathbf{y})) \right) d\mathbf{x}d\mathbf{y} .
    \end{align}
    Then, let $p_U(\mathbf{x}) \triangleq \sum_{i = 1}^m P_U(\mathbf{x}_i) \delta (\mathbf{x} - \mathbf{x}_i)$, with $P_U(\mathbf{x}_i) = \frac{1}{|\mathcal{T}_x|}$.
    \begin{align}
        \mathcal{J}(D) &= \int_{\mathcal{T}_y} \int_{\mathcal{T}_x} \sum_{i = 1}^m P_X(\mathbf{x}_i) \delta (\mathbf{x} - \mathbf{x}_i) p_{Y|X}(\mathbf{y}|\mathbf{x}) r(D(\mathbf{x}, \mathbf{y})) - \sum_{i = 1}^m \delta (\mathbf{x} - \mathbf{x}_i) p_Y(\mathbf{y}) f^{*} \left( r(D(\mathbf{x},\mathbf{y})) \right) d\mathbf{x}d\mathbf{y} ,
    \end{align}
    where $f_{u}^{*} \left( r(D(\mathbf{x},\mathbf{y})) \right) = |\mathcal{T}_x| f^{*} \left( r(D(\mathbf{x},\mathbf{y})) \right)$.
    Then, by using the indicator property of the delta function, the objective function becomes 
    \begin{align}
    \label{eq:supervised_general_integral_of_sum}
        \mathcal{J}(D) &= \int_{\mathcal{T}_y} \sum_{i = 1}^m P_X(\mathbf{x}_i) p_{Y|X}(\mathbf{y}|\mathbf{x}_i) r(D(\mathbf{x}_i,\mathbf{y})) d\mathbf{y} - \int_{\mathcal{T}_y} p_Y(\mathbf{y}) \sum_{i = 1}^m f^{*} \left( r(D(\mathbf{x}_i,\mathbf{y})) \right) d\mathbf{y} ,
    \end{align}
    that is equivalent to \eqref{eq:supervised_general_value_function_appendix}.

    
\end{proof}

\subsection{Proof of Closed Form Posterior in the Exponential Case}
\label{subsec:appendix closed form exponential}
\textbf{Closed Form Posterior Exponential Case}: Let define the model $Y= X+N$, where $X \sim exp(\lambda)$, $N \sim exp(\lambda)$. Therefore, $Y \sim \Gamma(2, \frac{1}{\lambda})$ (i.e., a Gamma distribution with shape $2$ and rate $\frac{1}{\lambda}$). 
Let define $Y|X \triangleq Z$, then the cumulative density function (CDF) is
\begin{align}
    P\left[ Z \leq z \right] &= P \left[ N+x_c \leq z \right] = P\left[ N \leq z - x_c \right] \notag \\
    &= \left( 1-e^{-\lambda(z-x_c)} \right)\mathds{1}(z-x_c) ,
\end{align}
where $X_c$ is constant.
By deriving the CDF w.r.t. $z$, we compute the likelihood
\begin{equation}
    p_{Y|X}(y|x) = \lambda e^{-\lambda(y|x - x_c)} \mathds{1}(y|x - x_c) = p_N(n).
\end{equation}
Then, the posterior probability is computed as
\begin{align}
\label{eq:posterior_closed_form_exponential}
    p_{X|Y}(x|y) &= \frac{p_N(n)p_X(x)}{p_Y(y)} \notag \\
    &= \frac{1}{y} \mathds{1}(x) \mathds{1}(y) .
\end{align}

\subsection{Proof of Closed Form Posterior in the Gaussian Case}
\label{subsec:appendix closed form gaussian}
\textbf{Closed Form Posterior Gaussian Case}: Let define the model $Y = X+N$, where $X \sim \mathcal{N}(0, \Sigma_X)$, $N \sim \mathcal{N}(0, \Sigma_N)$. Thus, $Y|X \sim \mathcal{N}(X, \Sigma_N)$, $Y \sim \mathcal{N}(0, \underbrace{\Sigma_N + \Sigma_X}_{\triangleq \Sigma_Y})$. 

The posterior probability is
\begin{align}
p_{X|Y}(x|y) &= \frac{p_{Y|X}(y|x) p_{X}(x)}{p_{Y}(y)} \\
&= \frac{ \frac{1}{\sqrt{(2 \pi)^d |\Sigma_N|}} e^{-\frac{1}{2} (y-x)^T \Sigma_N^{-1} (y - x) } }{\frac{1}{\sqrt{(2 \pi)^d |\Sigma_Y|}} e^{-\frac{1}{2} y^T \Sigma_Y^{-1} y }} \notag \\ 
& \cdot \frac{1}{\sqrt{(2 \pi)^d |\Sigma_X|}} e^{-\frac{1}{2} x^T \Sigma_X^{-1} x }\\ 
&= \frac{1}{\sqrt{(2 \pi)^d}} \sqrt{\frac{|\Sigma_Y|}{|\Sigma_N||\Sigma_X|}} \exp\left\{-\frac{1}{2}G\right\} ,
\end{align}
where
\begin{align}
G = x^T \Sigma_x^{-1} x + y^T \Sigma_N^{-1} y + x^T \Sigma_N^{-1} x - 2 x^T \Sigma_N^{-1}y - y^T \Sigma_Y^{-1} y . 
\end{align}

In the scalar case, where $d=1$ and $\Sigma_N = \sigma_N^2$, $\Sigma_X = \sigma_X^2$ and $\Sigma_Y = \sigma_X^2 + \sigma_N^2$, we have:
\begin{align}
G &= \frac{x^2}{\sigma_X^2} + \frac{y^2}{\sigma_N^2} + \frac{x^2}{\sigma_N^2} - 2 \frac{xy}{\sigma_N^2} - \frac{y^2}{\sigma_y^2} \\
&= x^2 \left( \frac{1}{\sigma_X^2} + \frac{1}{\sigma_N^2} \right) + y^2 \left( \frac{1}{\sigma_N^2} - \frac{1}{\sigma_Y^2}\right) - \frac{2xy}{\sigma_N^2} \\
&= x^2 \left( \frac{\sigma_X^2 + \sigma_N^2}{\sigma_X^2\sigma_N^2} \right) + y^2 \left( \frac{\sigma_Y^2 - \sigma_N^2}{\sigma_N^2\sigma_Y^2}\right) - \frac{2xy}{\sigma_N^2} \\
&= \frac{1}{\sigma_N^2} \left[ x^2 \frac{\sigma_N^2 + \sigma_X^2}{\sigma_X^2}  + y^2 \frac{\sigma_x^2}{\sigma_X^2 + \sigma_N^2} - 2xy \right] \\
&= \frac{1}{\sigma_N^2} \left[ x^2 k^2  + \frac{y^2}{k^2} - 2xy \right] \\
&= \frac{1}{\sigma_N^2} \left[ kx - \frac{y}{k} \right]^2 ,
\end{align}
where $k \triangleq \frac{\sigma_X^2 + \sigma_N^2}{\sigma_X^2} > 1$. \\
Then, in the scalar case the posterior density expression becomes 
\begin{align}
\label{eq:posterior_closed_form_gaussian}
    p_{X|Y}(x|y) = \frac{1}{\sqrt{2 \pi}} \sqrt{\frac{\sigma_Y^2}{\sigma_N^2 \sigma_X^2}} e^{-\frac{1}{2 \sigma_N^2} \left[ kx - \frac{y}{k} \right]^2}.
\end{align}

\section{Appendix: Bottom-up Approach}
\label{subsec:appendix_insights_bottom_up}
In this section, we provide additional insights on the bottom-up approach proposed in Theorem \ref{theorem:Bottom-up}. The idea is to design the objective function by starting from the desired discriminator's output. When designing an objective function using the bottom-up approach, the choice of $k(\cdot)$ is the fundamental starting point of the procedure, which is summarized in the following.
\begin{itemize}
    \item $k(\cdot)$ is the first DOF and must be chosen as a deterministic and invertible function in the domain $x>0$. Some examples are $k(x)=x$ and $k(x)=1/(1+x)$. Then the posterior pdf must be expressed as the density ratio $p_{XY}/p_Y$.
    \item $g_1(\cdot)$ is an additional DOF and can be chosen as defined in Corollary \ref{corollary:bottom-up-top-down-value-functions}. 
    \item Substitute $k(\cdot)$ and $g_1(\cdot)$ in \eqref{eq:g(D)}.
    \item Compute the integral of \eqref{eq:g(D)} w.r.t. $D$, which leads to $\tilde{\mathcal{J}}(D)$.
    \item The integral in \eqref{eq:value_function_general_bottom_up} represents the expectation computed over the supports $\mathcal{T}_x$ and $\mathcal{T}_y$.
\end{itemize}
This exact procedure is used to prove Theorem \ref{theorem:new_f}.\\
In the case of the KL-based objective function that we proposed, to ensure that \eqref{eq:bottom_up_posterior_estimator} is a density for the case of discrete $X$, we use the softmax as the last activation function of the discriminator. However, $k(\cdot)$ is not the softmax function because if it were the softmax, then $D^\diamond(\mathbf{x}, \mathbf{y}) = softmax(p_{X|Y}(\mathbf{x}|\mathbf{y}))$, where $p_{X|Y}(\mathbf{x}|\mathbf{y})$ is the density (i.e., the output of the softmax will be $p_{X|Y}(\mathbf{x}|\mathbf{y})$ itself, thus we do not need to apply additional transformations). Therefore, in the case of the KL divergence, $k^{-1}(x)=x$. 

\section{Appendix: Additional Numerical Results}
\label{subsec:appendix_numerical_results}
\subsection{Image Datasets}
We study the performance of the proposed objective functions for classification (listed in Appendix \ref{subsec:Appendix_value_functions}) for four image classification datasets. \textbf{MNIST} \cite{lecun1998gradient} comprises 10 classes with 60,000 images for training and 10,000 images for testing. \textbf{Fashion MNIST} has 10 classes with 60,000 images for training and 10,000 images for testing. \textbf{CIFAR10} \cite{krizhevsky2009learning} has 10 classes
with 50,000 images for training and 10,000 images for testing. \textbf{CIFAR100} \cite{krizhevsky2009learning} has 100 classes with 50,000 images for training and 10,000 images for testing. 
Table \ref{tab:classification_images_appendix} shows the classification accuracy of different network architectures trained by using the supervised versions of the objective functions listed in Table \ref{tab:value functions} and the one presented in Theorem \ref{theorem:new_f} (see Appendix \ref{subsec:Appendix_value_functions}). The accuracies are in the form $XX \pm YY$, where $XX$ and $YY$ represent the mean and standard deviation, respectively, obtained over multiple runs of the code. The tests are run by using various architectures: VGG \cite{DBLP:journals/corr/SimonyanZ14a}, DLA\footnote{The implementation of the DLA is a simplified version of the one presented in the original paper} \cite{yu2018deep}, ResNet18 \cite{resnet}, DenseNet \cite{densenet}, PreActResNet \cite{he2016identity}, MobileNetV2 \cite{sandler2018mobilenetv2}. Interestingly, from Tab. \ref{tab:classification_images_appendix}, it can be observed that every network architecture suits some objective functions more than others. For instance, the PreActResNet attains the highest accuracy when trained with the GAN-based objective function (even if the difference in performance is just 0.2\% different from the second-best one). Differently, the DenseNet performs optimally in CIFAR10 and CIFAR100 when trained with the KL-based objective function, even if the difference in accuracy w.r.t. the SL-based objective function is minimal. 
The MobileNetV2 obtains the best performance when trained with the SL-based divergence, and the difference with the second-best objective function is around $2\%$, which is significant. 
The choice of the architecture often depends on the goal of the classification algorithm. For embedded systems, light architectures are used. Therefore, the MobileNetV2 is an option. In such a case, the SL divergence is the preferred choice for the network's training.
\setcounter{table}{4}
\begin{table*}
\caption{Classification accuracy on MNIST (M), Fashion MNIST (FM), CIFAR10 (C10), and CIFAR100 (C100). The PreActResNet is referred to as PAResNet, while the MobileNetV2 is referred to as MobileNet.} 
  \begin{center}
  \begin{small}
  \begin{sc}
    \begin{tabular}{ c c c c c c c c } 
     \hline
     Dataset & Model & CE & RKL & HD & GAN & P & SL \\
     \hline
     M & Shallow & $\textbf{99.08} \pm 0.06$ & $96.05 \pm 0.25$ & $98.68 \pm 0.05$ & $\textbf{99.08} \pm 0.07$ & $98.89 \pm 0.08$ & $99.03 \pm 0.04$ \\
     \hline
     FM & Shallow & $91.64 \pm 0.09$ & $82.63 \pm 1.78$ & $90.75 \pm 0.13$ & $91.63 \pm 0.10$ & $89.86 \pm 0.67$ & $\textbf{91.83} \pm 0.02$ \\
     \hline
      \multirow{7}{*}{C10} & Shallow & $70.13 \pm 0.05$ & $63.59 \pm 0.34$ & $69.38 \pm 0.28$ & $69.98 \pm 0.15$ & $59.62 \pm 0.45$ & $\textbf{70.87} \pm 0.26$  \\
      & VGG & $93.69 \pm 0.03$ & $84.24 \pm 2.21$ & $93.51 \pm 0.06$ & $93.75 \pm 0.04$ & $84.79 \pm 0.21$ & $\textbf{93.93} \pm 0.08$\\
      & DLA & $95.04 \pm 0.02$ & $90.83 \pm 0.10$ & $94.56 \pm 0.11$ & $95.04 \pm 0.13$ & $91.61 \pm 0.21$ & $\textbf{95.31} \pm 0.09$\\
      & ResNet & $95.39 \pm 0.04$ & $92.88 \pm 0.26$ & $95.15 \pm 0.08$ & $95.24 \pm 0.06$ & $93.78 \pm 0.21$ & $\textbf{95.43} \pm 0.04$\\
      & DenseNet & $\textbf{95.82} \pm 0.06$ & $91.52 \pm 0.14$ & $94.93 \pm 0.07$ & $95.67 \pm 0.02$ & $94.34 \pm 0.16$ & $95.53 \pm 0.13$  \\
      & PAResNet & $94.30 \pm 0.15$ & $89.43 \pm 0.39$ & $56.44 \pm 0.02$ & $\textbf{95.17} \pm 0.07$ & $86.24 \pm 0.08$ & $95.09 \pm 0.02$ \\
      & MobileNet & $92.59 \pm 0.13$ & $83.97 \pm 0.21$ & $91.95 \pm 0.33$ & $92.37 \pm 0.14$ & $84.30 \pm 0.32$ & $\textbf{93.89} \pm 0.15$ \\
      \hline
      \multirow{6}{*}{C100} & VGG & $72.73 \pm 0.30$ & $45.80 \pm 2.86$ & $73.51 \pm 0.03$ & $68.88 \pm 0.20$ & $37.19 \pm 0.66$ & $\textbf{73.61} \pm 0.05$\\
      & DLA & $76.29 \pm 0.43$ & $68.86 \pm 1.17$ & $78.63 \pm 0.14$ & $77.34 \pm 0.22$ & $57.97 \pm 0.07$ & $\textbf{78.65} \pm 0.01$\\
      & ResNet & $\textbf{78.29} \pm 0.18$ & $ 70.68 \pm 0.44$ & $77.59 \pm 0.06$ & $77.43 \pm 0.08$ & $61.12 \pm 0.23$ & $78.03 \pm 0.04$\\
      & DenseNet & $\textbf{80.09} \pm 0.02$ & $72.03 \pm 0.23$ & $79.91 \pm 0.01$ & $79.32 \pm 0.34$ & $62.27 \pm 0.21$ & $80.03 \pm 0.02$\\
      & PAResNet & $77.13 \pm 0.15$ & $59.28 \pm 0.12$ & $76.98 \pm 0.05$ & $\textbf{77.39} \pm 0.18$ & $61.64 \pm 0.11$ & $77.19 \pm 0.25$ \\
      & MobileNet & $72.61 \pm 0.08$ & $53.17 \pm 0.35$ & $73.00 \pm 0.30$ & $65.66 \pm 0.46$ & $46.00 \pm 0.37$ & $\textbf{74.78} \pm 0.23$ \\
     \hline
    \end{tabular}
    \end{sc}
    \end{small}
    \end{center}
    \vskip -0.1in
    \label{tab:classification_images_appendix}
\end{table*}

The SL divergence has been proved to have favourable convergence conditions in Corollary \ref{corollary:unsupervised_GAN_SL}. We numerically compare the convergence speed of the GAN and SL divergences on the CIFAR10 dataset in Figure \ref{fig:speed_accuracy}, demonstrating the effectiveness of Corollary \ref{corollary:unsupervised_GAN_SL}. The accuracy behavior is attained by averaging over multiple runs of the code.
In Figure \ref{fig:speed_accuracy}, the accuracy over the test dataset is showed, for each training epoch, with a semi-transparent color (blue and orange for the GAN and SL divergences, respectively). The vivid colors represent a moving average over $5$ epochs, helping the clarity of the visualization. Quantitatively, we report in Tab. \ref{tab:speed_accuracy} the difference between the accuracy obtained by training the network with the SL and GAN divergences. In detail, each column is identified by a number $k$ and contains, for various discriminator architectures, the quantity $A^k_{SL} - A^k_{GAN}$, where $A^k_{f}$ is the average accuracy over multiple runs of the code for the selected $f$-divergence, in the epochs interval $[1, k]$. Each cell in Table \ref{tab:speed_accuracy} contains a positive value, which shows the average faster convergence property of the SL divergence. Furthermore, the difference in speed convergence is more significant for the VGG training than for the ResNet18 and DLA architectures. 
\begin{figure}[htbp]
	\centerline{\includegraphics[width=\columnwidth]{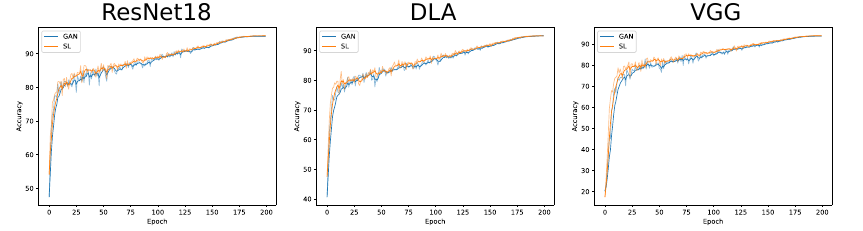}}
	\caption{Convergence speed of the test accuracy over $200$ training epochs.}
	\label{fig:speed_accuracy}
\end{figure}
 \begin{table}
\caption{Difference in speed of convergence of the classification accuracy on CIFAR10, between the SL and GAN divergences. Each column is characterized by a number $k$, which represents the interval $[1,k]$ of epochs over which the average is computed. Each cell contains the difference of accuracy $A^k_{SL} - A^k_{GAN}$.} 
\centering
\vskip 0.15in
  \begin{center}
  \begin{small}
  \begin{sc}
    \begin{tabular}{ c | c | c | c | c | c | c | c | c | c | c } 
     \hline
     Model & 5 & 10 & 25 & 50 & 75 & 100 & 125 & 150 & 175 & 200 \\
     \hline
     \hline
      ResNet18  & $4.15$ & $2.81$ & $1.48$ & $1.43$ &$1.16$ &$1.12$ & $0.95$ & $0.86$ & $0.75$ & $0.67$ \\
      \hline
      DLA & $5.40$ & $4.39$ & $2.52$ & $1.43$ & $1.15$ & $1.11$ & $1.00$ & $0.91$ & $0.81$ & $0.72$\\
      \hline
      VGG & $5.93$ & $7.40$ & $5.19$ & $3.42$ & $2.68$ &$2.36$ & $2.07$ &$1.84$ & $1.63$ & $1.44$ \\
     \hline
    \end{tabular}
    \end{sc}
    \end{small}
    \end{center}
    \vskip -0.1in
    \label{tab:speed_accuracy}
\end{table}

The speed of convergence and the upper-boundness of $D_{SL}$ is (see Corollaries \ref{corollary:unsupervised_GAN_SL} and \ref{corollary:upper_lower_bound_newF} in Appendix \ref{subsec:appendix_corollary_speed_convergence} and \ref{subsec:appendix_corollary_upper_lower_bound_SL}, respectively) confirm the utility of the SL divergence for double optimization problems. For instance, when the objective function of the learning algorithm is formulated as a max-max game. For example, in \cite{wei2020optimizing} the $f$-mutual information (which is estimated by formulating a maximization problem) is maximized for classification with noisy labels. Similarly, in \cite{hjelm2018learning} the authors maximize the JS-mutual information (where JS refers to the Jensen-Shannon divergence, which is equivalently used in this paper as GAN divergence) for representation learning applications (then the model is extended in \cite{NEURIPS2019_ddf35421}). In \cite{letizia2023cooperative}, the authors formulate a mutual information maximization algorithm to achieve the channel capacity in a data communication system. In \cite{zhu2021Arbitrary}, the authors maximize the JS-mutual information for the generation of talking faces.

\subsection{Additional Decoding Tasks}

\subsubsection{AWGN}
We analyze the decoding task in presence of additive white Gaussian noise (AWGN) in the communication channel. 
Let $X$ be a $d$-dimensional binary vector, and $N \sim \mathcal{N}(0,\Sigma_N)$ be Gaussian noise, with $\Sigma_N$ diagonal. Let $Y=X+N$ be the output of the communication channel. 
The SER behavior when varying the SNR is shown in Fig. \ref{fig:AWGN_latent_dim6} for each objective function analyzed. To compare the estimated SER, we visualize the SER achieved by the maxL decoder, which corresponds to the optimal decoder for an AWGN channel with uniform distribution $p_X(\mathbf{x})$ \cite{proakis2007fundamentals}.
The proposed SL divergence achieves the best performance and close to the optimal maxL decoder. In general, different objective functions perform better than $\mathcal{J}_{KL}(D)$.\\
\begin{figure}[htbp]
	\centerline{\includegraphics[width=0.7\columnwidth]{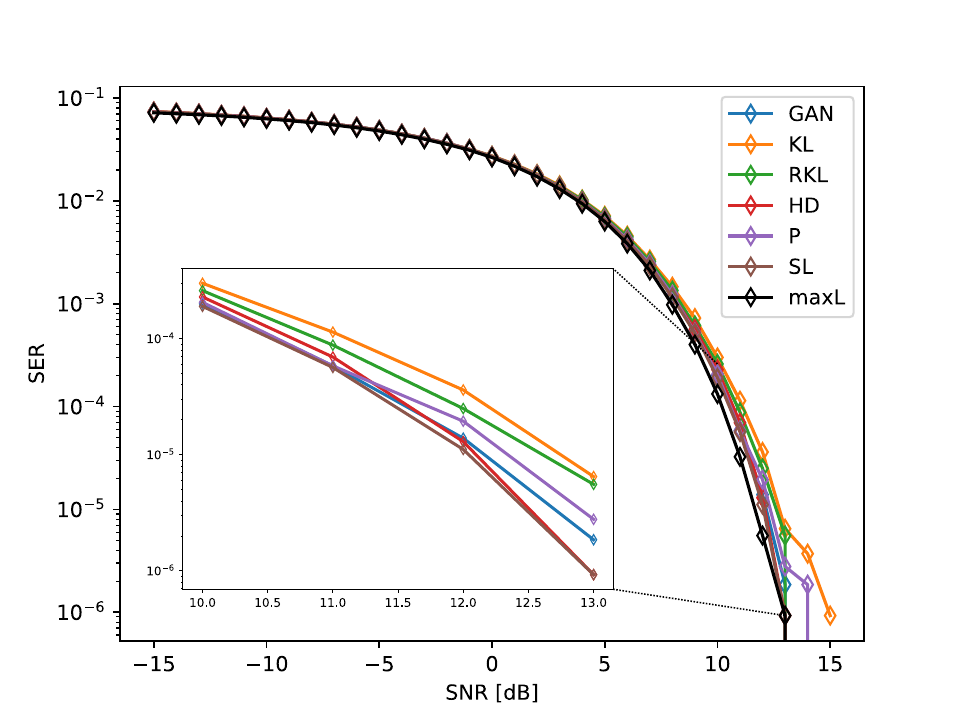}}
	\caption{SER achieved in an AWGN channel by the proposed posterior probability estimators.}
	\label{fig:AWGN_latent_dim6}
\end{figure}

\subsubsection{PAM with Non-Uniform Source}
Similarly to Section \ref{sec:Results}, we consider a 4-PAM transmission. However, in this case we examine the case where the symbols transmitted do not have a uniform prior probability $p_X(\mathbf{x})$. We define the alphabet of $X$ to be $\mathcal{A}_x = \{ x _1, x_2, x_3, x_4\}$ with probabilities $P(x_1) = P(x_2) = P/2$ and $P(x_3) = P(x_4) = (1-P)/2$, where $P=0.05$. The SER for various values of SNR is reported in Fig. \ref{fig:pam_triangular}, where the discriminators trained with the supervised versions of the objective functions in Tab. \ref{tab:value functions} and in \eqref{eq:fNOME_cost_fcn} (see Appendix \ref{subsec:Appendix_value_functions}) are compared with the maxL and MAP decoders. Although the maxL decoder is optimal for an AWGN channel, the extreme non-uniformity of the channel significantly impacts its performance. Differently, the optimal MAP decoder knows the distribution of $p_X(\mathbf{x})$. The discriminator trained with the SL and GAN divergences achieves performance close to the optimal MAP decoder.

\begin{figure}[htbp]
	\centerline{\includegraphics[width=0.7\columnwidth]{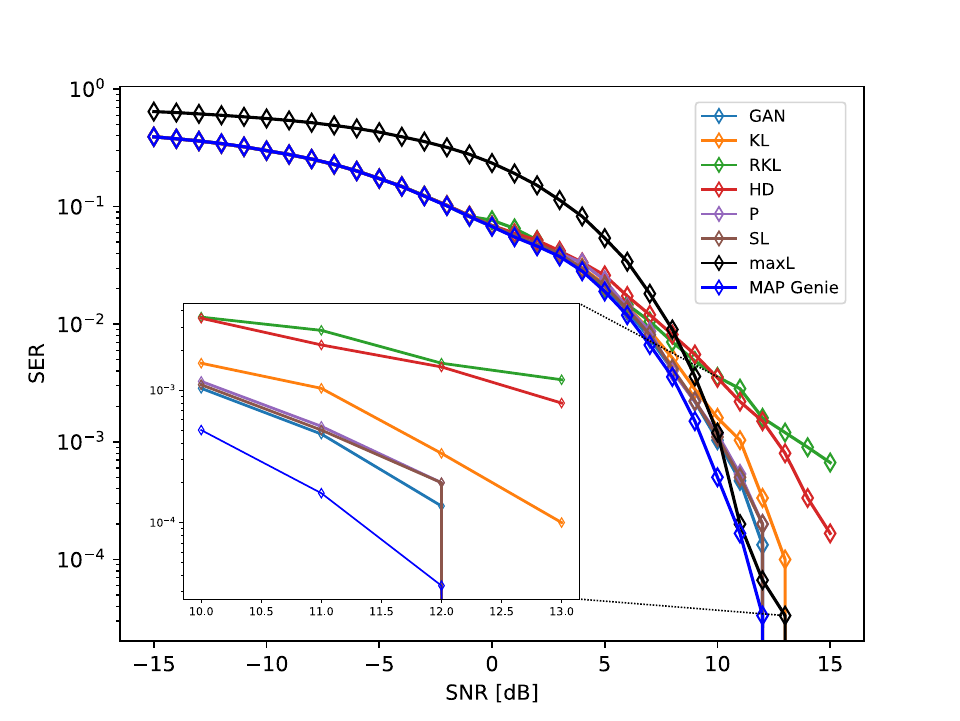}}
	\caption{SER achieved in an AWGN channel by a 4-PAM with non-uniform source probability distribution.}
	\label{fig:pam_triangular}
\end{figure}

\end{document}